%% file: main.tex
\documentclass{article}
\usepackage{adjustbox}
\usepackage{tcolorbox}
\usepackage{xcolor}
\usepackage{microtype}
\usepackage{graphicx}
\usepackage{subcaption}
\usepackage{booktabs} % for professional tables
\definecolor{lightblue}{rgb}{0.6, 0.8, 0.9}
\definecolor{darkblue}{rgb}{0.2,0.4,0.6}
\definecolor{darkgreen}{rgb}{0, 0.55, 0.12}
\definecolor{darkred}{rgb}{0.6,0,0}
\input{math_commands.tex}
\usepackage{algorithmic}
\usepackage{algorithm}
\usepackage{arydshln}

\usepackage{hyperref}
\let\OriginalAddContentsLine\addcontentsline

\usepackage[preprint]{icml2026}

\usepackage{amsmath}
\usepackage{amssymb}
\usepackage{mathtools}
\usepackage{amsthm}
\usepackage{enumitem}
\usepackage{bm}
\usepackage{url}

\usepackage{pifont}
\usepackage{booktabs}
\usepackage{multicol}
\usepackage{multirow}
\usepackage{listings}
\usepackage{balance}
\usepackage{needspace}
\usepackage{xspace}
\usepackage{colortbl}
\usepackage{setspace}
\usepackage{fancybox}
\usepackage{tocloft}
\usepackage{etoc}

\usepackage[capitalize, noabbrev]{cleveref}

\theoremstyle{plain}
\input{macro-math.tex}

\theoremstyle{definition}

\theoremstyle{remark}

\usepackage[textsize=tiny]{todonotes}

\icmltitlerunning{FreeAct: Freeing Activations for LLM Quantization}

\begin{document}
  \twocolumn[ \icmltitle{FreeAct: \underline{Free}ing \underline{Act}ivations for LLM Quantization}



  \icmlsetsymbol{equal}{*}

  \begin{icmlauthorlist}
    \icmlauthor{Xiaohao Liu}{nus} 
    \icmlauthor{Xiaobo Xia}{nus} 
    \icmlauthor{Manyi Zhang}{huawei}
    \icmlauthor{Ji-Fu Li}{huawei}
    \icmlauthor{Xianzhi Yu}{huawei}\\
    \icmlauthor{Fei Shen}{nus}
    \icmlauthor{Xiu Su}{csu}
    \icmlauthor{See-Kiong Ng}{nus} 
    \icmlauthor{Tat-Seng Chua}{nus}
  \end{icmlauthorlist}

  \icmlaffiliation{nus}{National University of Singapore} \icmlaffiliation{huawei}{Huawei Technology} \icmlaffiliation{csu}{Central South University}

  \icmlcorrespondingauthor{Xiaobo Xia}{xbx@nus.edu.sg} 

  \icmlkeywords{Machine Learning, ICML}

  \vskip 0.3in ]

  \printAffiliationsAndNotice{} 

  \begin{abstract}
    Quantization is pivotal for mitigating the significant memory and computational overhead of Large Language Models (LLMs). While emerging transformation-based methods have successfully enhanced quantization by projecting feature spaces onto smoother manifolds using orthogonal matrices, they typically enforce a rigid \textit{one-to-one} transformation constraint. This static approach fails to account for the dynamic patterns inherent in input activations, particularly within diffusion LLMs~(dLLMs) and Multimodal LLMs~(MLLMs), where varying token types exhibit distinct distributions. To advance this, we propose FreeAct, a novel quantization framework that relaxes the static one-to-one constraint to accommodate dynamic activation disparities. Theoretically, we leverage the rank-deficient nature of activations to derive a solution space that extends beyond simple inverse matrices, enabling the decoupling of activation transformations from weights. Methodologically, FreeAct identifies token-specific dynamics (\textit{i.e.}, vision \textit{v.s.} text, or masked tokens) and allocates distinct transformation matrices to the activation side, while maintaining a unified, static transformation for the weights. Extensive experiments across dLLMs and MLLMs demonstrate that FreeAct significantly outperforms baselines, up to 5.3\% performance improvement, with in-depth analyses. Our code will be publicly released.
  \end{abstract}

  \input{sections/introduction}
  \input{sections/background}
  \input{sections/method}

\input{sections/experiments}
  \input{sections/discussions}

  \bibliography{reference}
  \bibliographystyle{icml2026}

  \renewcommand{\cftsecfont}{\normalsize}
  \renewcommand{\cftsubsecfont}{\normalsize}
  \renewcommand{\cftbeforesecskip}{12pt}
  \renewcommand{\cftbeforesubsecskip}{12pt}

  \addtocontents{toc}{\protect
  \setcounter{tocdepth}{2}}

  \newpage
  \appendix
  \onecolumn
  \begin{center}
    \LARGE \textbf{Appendix}
    \vspace{1em}
  \end{center}
  \tableofcontents
  \let
  \addcontentsline\OriginalAddContentsLine
  {\newpage}{\input{sections/appendix}}
\end{document}

%% file: math_commands.tex
\usepackage{amsmath,amsfonts,bm}

\def\eqref#1{equation~\ref{#1}}
\def\1{\bm{1}}

\def\vzero{{\bm{0}}}

\def\vc{{\bm{c}}}

\def\vm{{\bm{m}}}

\def\vr{{\bm{r}}}
\def\vs{{\bm{s}}}

\def\vx{{\bm{x}}}

\def\mA{{\bm{A}}}
\def\mB{{\bm{B}}}
\def\mC{{\bm{C}}}

\def\mI{{\bm{I}}}

\def\mP{{\bm{P}}}

\def\mU{{\bm{U}}}
\def\mV{{\bm{V}}}
\def\mW{{\bm{W}}}
\def\mX{{\bm{X}}}
\def\mY{{\bm{Y}}}
\def\mZ{{\bm{Z}}}

\DeclareMathAlphabet{\mathsfit}{\encodingdefault}{\sfdefault}{m}{sl}
\SetMathAlphabet{\mathsfit}{bold}{\encodingdefault}{\sfdefault}{bx}{n}

\def\gI{{\mathcal{I}}}

\def\gM{{\mathcal{M}}}

\def\gQ{{\mathcal{Q}}}

\def\gT{{\mathcal{T}}}

\def\sI{{\mathbb{I}}}

\def\sR{{\mathbb{R}}}

\newcommand{\E}{\mathbb{E}}
\newcommand{\Ls}{\mathcal{L}}

\newcommand{\diag}[1]{\text{diag}(#1)}

\definecolor{item}{RGB}{15, 86, 157}
\definecolor{query}{RGB}{169, 16, 3}

\definecolor{-}{rgb}{0.25,0.41,0.88}
\definecolor{+}{rgb}{0.70,0.13,0.13}

\definecolor{table_color}{RGB}{239,246,251}
\definecolor{good}{RGB}{58,113,104}
\definecolor{bad}{RGB}{180,0,0}
\definecolor{malboxborder}{RGB}{180,0,0}
\definecolor{benboxborder}{RGB}{81,91,131}
\definecolor{malboxbg}{RGB}{255,250,250}
\definecolor{benboxbg}{RGB}{251,252,254}
\definecolor{titlebg}{RGB}{77,77,77}
\newsavebox{\CaseStudy}
\newsavebox{\CaseStudyApdxOne}
\newsavebox{\CaseStudyApdxTwo}
\newsavebox{\CaseStudyApdxThree}
\newsavebox{\DSRPrompt}
\newsavebox{\CRPrompt}
\newsavebox{\GPTPrompt}

\newtcolorbox{gptpromptbox}[1][]{%
  colback=white,
  colframe=black,
  boxrule=1pt,
  arc=5pt,
  width=\textwidth,
  boxsep=2pt,
  left=2pt,
  right=2pt,
  top=2pt,
  bottom=2pt,
  title={PROMPT:}, % Title with colon
  fonttitle=\large\bfseries, % Larger and bold font
  coltitle=white, % Title text color
  colbacktitle=titlebg, % Gray background for the title
  #1
}

\newtcolorbox{promptbox}[1][]{%
  colback=white,
  colframe=black,
  boxrule=1pt,
  width=\textwidth,
  arc=5pt,
  boxsep=2pt,
  left=2pt,
  right=2pt,
  top=2pt,
  bottom=2pt,
  #1
}

\newtcolorbox{baselinebox}{
  colback=darkblue!2,
  colframe=darkblue!20,
  boxrule=0.6pt,
  arc=3pt,
  left=1pt, right=1pt, top=1pt, bottom=1pt,
}

\newtcolorbox{freequantbox}{
  colback=darkgreen!5,
  colframe=darkgreen!50,
  boxrule=0.6pt,
  arc=3pt,
  left=1pt, right=1pt, top=1pt, bottom=1pt,
}

%% file: macro-math.tex
\allowdisplaybreaks

\newtheorem{thm}{Theorem}
\newtheorem{lem}[]{Lemma}
\newtheorem{prop}[thm]{Proposition}

\theoremstyle{definition}

%% file: sections/introduction.tex
\section{Introduction}

\begin{figure}[t]
    \centering
    \includegraphics[width=\linewidth]{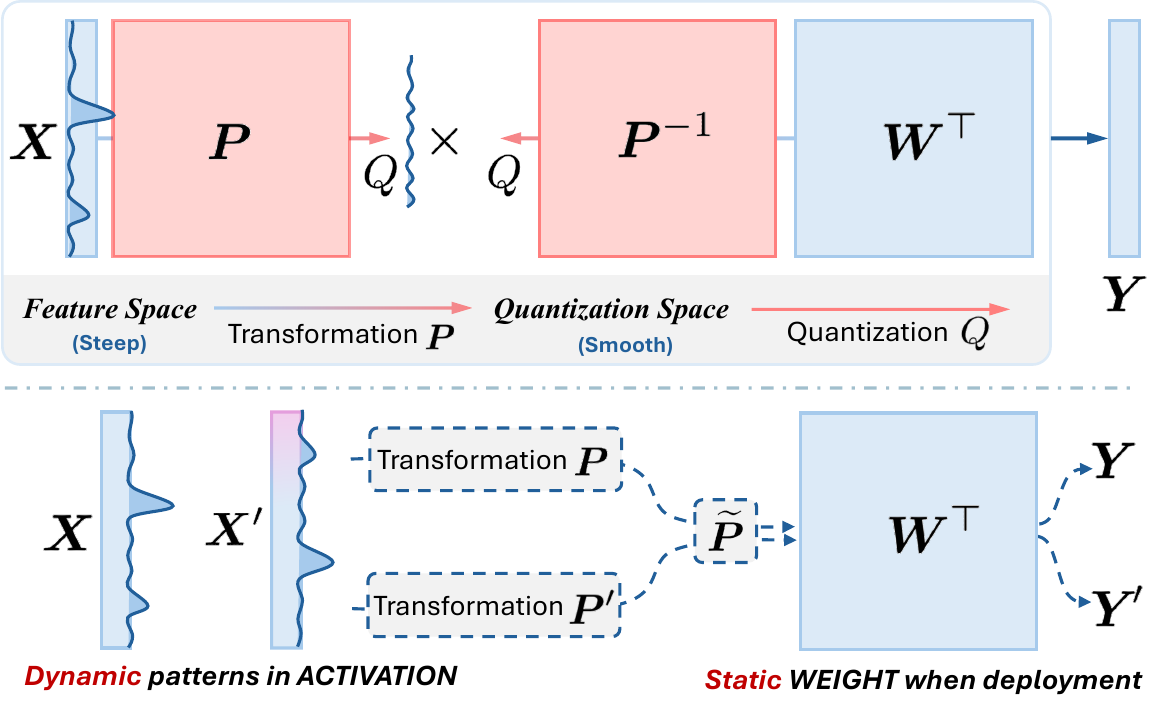}
    \caption{The steep activation is smoothed by the designed transformation matrix for quantization (\textit{top}). One weight corresponds to one transformation to ensure equivalence, \textit{i.e.}, $\mP\times\mP^{-1}= \mI$. In practice, activations are diverse, reflecting dynamic patterns, necessitating flexible transformation matrices on the activation side, while keeping static on the weight side for quantization (\textit{bottom}). }
    \vspace{-4mm}
    \label{fig:top}
\end{figure}

Large Language Models (LLMs) demonstrate remarkable performance across diverse tasks, enabled by pre-training on massive corpora with vast parameter counts~\cite{achiam2023gpt, touvron2023llama, grattafiori2024llama, yang2025qwen3, liu2024deepseek, guo2025deepseek}. However, this versatility and generalization come at the cost of significant memory and computation overhead~\cite{kaplan2020scaling,zhou2024survey,alizadeh2024llm,liu2025mtp, dai2025flashdecoding++next}. To maintain capabilities learned from the training stage, yet breaking the efficiency bottlenecks with lower overhead on deployment, recent research pivots toward one critical direction, LLM Quantization~\cite{egashira2024exploiting, tseng2024quip,xia2024efficient, dettmers2022gpt3}.

The fundamental goal of \textit{quantization} is to reduce the precision of model parameters and activations (\textit{e.g.}, from BF16 to INT4)~\cite{lin2025qserve, liu2025paretoq}. However, reducing bit-width introduces quantization error, particularly when dealing with significant outliers or non-uniform distribution of values in activations or weights. Conventional methods manage this error by 1) applying clipping thresholds to limit numerical ranges, or 2) employing fine-grained scaling, such as per-token~\cite{zhang2025quant, xu2025dllmquant} or per-channel scales~\cite{xiao2023smoothquant, lin2024awq}. Recently, transformation-based methods have emerged as a powerful technique~\cite{chee2023quip, tseng2024quip, ashkboos2024quarot,lin2024duquant, liu2024spinquant, ma2024affinequant}, by introducing an orthogonal matrix $\mP$ per linear layer to project the steep feature space onto a smoother quantization space, as illustrated at the top of Figure~\ref{fig:top}. The inverse $\mP^{-1}$ guarantees the equivalence of computation. One matrix and its unique inverse balance both activation and weight sides within LLMs, \textit{i.e.}, one-to-one transformation. This line of research, exemplified by QuaRot~\cite{ashkboos2024quarot} and FlatQuant~\cite{sun2024flatquant}, offers a new perspective to achieve higher quantization quality\footnote{We review the literature related to this work in Appendix~\ref{sec:related_work}.}.

Nevertheless, these methods overlook the dynamic patterns in activations, assuming a consistent behavior that can be handled by the one-to-one transformation. Despite the weights remaining static during inference, input activations can exhibit significant variability, particularly in diffusion LLMs (dLLMs)~\cite{nie2025large, ye2025dream,wang2025diffusion} or Multimodal LLMs (MLLMs)~\cite{bai2025qwen2,chen2024internvl,guo2025seed1}, as shown in Figure~\ref{fig:time_aware} and Figure~\ref{fig:modal_aware}. The disparities between activations motivate studies on time-aware (dLLMs)~\cite{zhang2025quant,xu2025dllmquant} and modality-aware quantization (MLLMs)~\cite{yu2025mquant}. Unfortunately, constrained by the requirement of maintaining a unique inverse matrix to ensure computational equivalence with static weights, they relegate the handling of dynamics solely to the activation-side scaling, neglecting the potential of dynamic transformations.

We pioneer relaxing the static one-to-one transformation constraint in LLM quantization. As illustrated at the bottom of Figure~\ref{fig:top}, distinct transformation matrices $\{\mP, \mP'\}$ are allocated to the activation side to accommodate distinct data types, \textit{e.g.}, vision and text tokens, that exhibit varying dynamic patterns. Meanwhile, the weight side remains static, utilizing a single common transformation $\widetilde\mP$. This design frees the activation transformation from the weights, offering new flexibility needed to handle such dynamics and thereby further bolstering quantization quality.

To operationalize this insight, we propose \textbf{FreeAct}, a novel post-training quantization method designed to address dynamic activation patterns by breaking the static one-to-one transformation constraint. We first characterize the dynamics arising from diverse token types. Specifically, in dLLMs, we investigate timestep-related variations in activation ranges driven by the significant disparities between masked and unmasked tokens. Similarly, in MLLMs, we observe analogous distinctions between vision and text tokens. FreeAct unifies these paradigms under a common principle. Theoretically, we leverage the rank-deficient nature of activations~\cite{zhang2024magr, jeon2402l4q} to derive a solution space for transformation relationships that extends beyond simple inverses, confirming the feasibility of FreeAct. Methodologically, we index tokens by type and allocate distinct transformation matrices accordingly. For activations, we construct matrices with both shared and unique components, using zero-padding to prevent information entanglement between unique subspaces. Conversely, the weight-side matrix unifies these components while maintaining orthogonality across different subspace bases. We optimize these quantization parameters by minimizing the error \textit{w.r.t.} specific activation types. Furthermore, we provide theoretical guarantees for the equivalence and highlight the implementation simplicity and compatibility with additional enhancements. Extensive experiments on both dLLMs and MLLMs demonstrate that FreeAct achieves superior quantization quality, delivering up to 5.3\% relative improvement over state-of-the-art (SOTA) baselines. Further analysis validates the rationale and efficacy of our method. Before delving into details, we summarize our contributions as follows:
\begin{itemize}[leftmargin=*]
    \item We pioneer relaxing the existing transformation constraint of the activation side to allow for more flexible, dynamic handling of varying activation patterns. Based on this, we unify two advanced LLM paradigms, \textit{i.e.}, dLLMs and MLLMs, under one common principle for quantization. 
    \item We propose FreeAct, a novel post-training quantization framework, utilizing the rank-deficiency of activations to derive distinct transformation matrices for different token types, while maintaining a unified, static transformation for weights. This method effectively handles dynamic patterns through a subspace-based construction and zero-padding strategy with theoretical support.
    \item We conduct comprehensive experiments across both dLLMs and MLLMs to demonstrate the superior quantization quality of FreeAct and provide in-depth analysis to discuss its rational and effective design.
\end{itemize}

%% file: sections/background.tex
\section{Preliminaries}

\subsection{Quantization}

\textbf{Setups.} Quantization is effective in accelerating the linear inference in networks, especially LLMs.
\begin{align}
    \gQ(\mX) = \left\lfloor \mX/s_{\mX}\right\rceil,\quad s_{\mX}= \max(|\mX|)/q_{\max}.
\end{align}
For a linear layer, the computation can be approximated by
\begin{align}
    \mX\mW & \approx \gQ(\mX)\times \gQ(\mW) = s_{\mX}s_{\mW}\cdot \gQ_\mX\times \gQ_\mW.
\end{align}
We are focusing on \textbf{4-bit} quantization for both weight and activation, \textit{i.e.}, W4A4. Under this setting, the model generally fails to predict reasonable tokens or to follow the instructions~\cite{liu2025quantization}. The quantization error~\cite{nagel2106white, li2024svdquant} can be defined as
\begin{align}
    \E(\mX,\mW) = \|\mX\mW - \gQ(\mX)\times \gQ(\mW)\|_{F}.
\end{align}
However, the direct quantization is normally ineffective, especially for cases with \textit{disparities between activations} in Large Language Models (LLMs). 

\textbf{Disparities between activations.} The activations exhibit different distributions with different ranges, suggesting that the activations need flexible quantization rather than a fixed process that corresponds to the weight side. In this paper, we explore two representative LLMs: \textit{Diffusion LLMs} (dLLMs)~\cite{nie2025large, ye2025dream} and \textit{Multimodal LLMs} (MLLMs)~\cite{bai2025qwen2,chen2024internvl}.
\begin{itemize}[leftmargin=*]
    \item \textit{\textbf{Mask-aware.}} dLLMs predict multiple tokens at newly added \texttt{[MASK]} tokens, which are progressively determined along with the denoising process. From Figure~\ref{fig:time_aware}, we identify that there are significantly different activation distributions between masked and unmasked tokens.
        \begin{figure}[H]
        \centering
        \vspace{-3mm}
        \includegraphics[width=\linewidth]{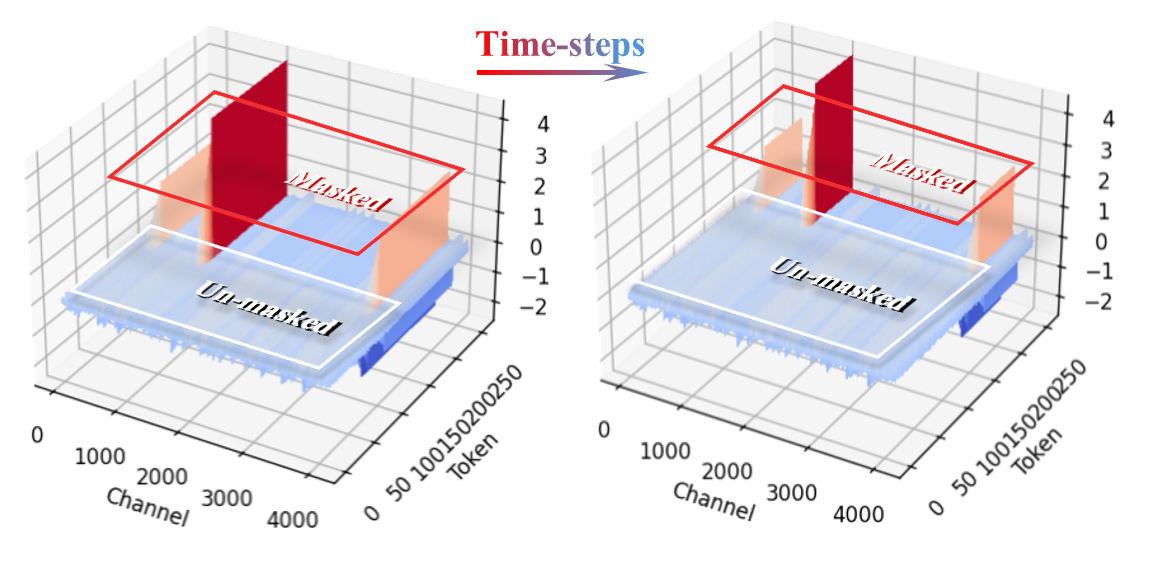}
        \caption{Illustrative example of activations between masked and unmasked tokens along the diffusion time-steps (dLLMs).}
        \label{fig:time_aware}
        \vspace{-5mm}
    \end{figure}
    \item \textit{\textbf{Modality-aware.}} MLLMs process different modalities, \textit{e.g.,} vision and text, as input. Different modality embeddings exhibit different activation distributions, as shown in Figure~\ref{fig:modal_aware}\footnote{Full visualization will be presented in Appendix~\ref{appendix:activation_vis} to save the main body space.}. This disparity, despite being diminished, will be passed to the following layers.
    \begin{figure}[H]
        \centering
        \vspace{-3mm}
        \includegraphics[width=\linewidth]{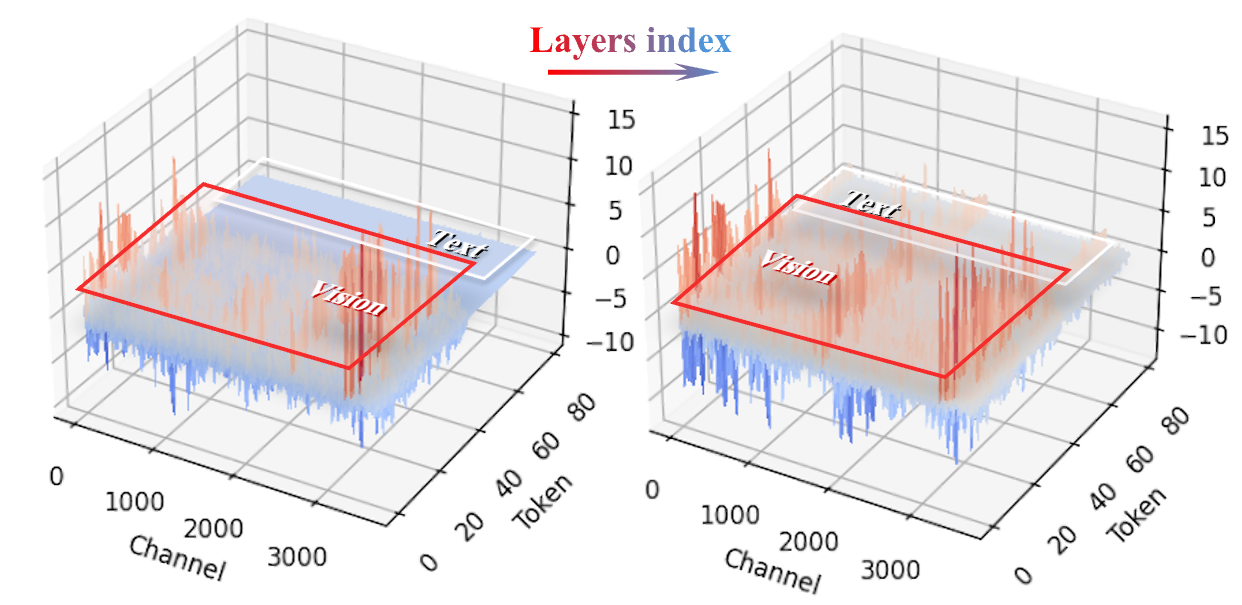}
        \caption{Illustrative example of activation distributions between vision and text tokens at different layers in MLLMs.}
        \label{fig:modal_aware}
    \end{figure}
\end{itemize}

From these preliminary investigations, we propose to process these two widely adopted paradigms (Multimodal and diffusion LLMs) under one quantization principle.

\subsection{Multimodal and Diffusion LLMs}

\textbf{Learning objective.} The learning of dLLMs aims to predict masked tokens with a unified cross-entropy loss \cite{yang2025mmada, nie2025large}: 
\begin{align*}
\Ls_{\gT}= -\E_{t,\vx_0, \vx_t}\Big[ \frac{1}{t}\sum_{l}^{L}\mathbb{I}(\vx_{t}^{l}= [\texttt{MASK}])\log p_{\theta}(\vx_{0}^{l}| \vx_{t}) \Big],
\end{align*} 
where $L$ denotes the sequence length of $\vx_{0}$ and $t$ is the time step sampled from $[0,1]$ uniformly, corresponding to $\vx_{t}$. $\mathbb{I}(\cdot)$ is the indicator.
MLLMs predict the next token, while accepting vision and text as input, formally, 
\begin{align}
\Ls_{\gM}= -\E_{\vx}\Big[ \sum_{l}^{L}\log p_{\theta}(\vx_{l+1}| \vx_{<l}) \Big].
\end{align} 
Here $\vx$ is modal-agnostic for simplicity.

\textbf{Inference.} For dLLMs, given one prompt $p_{0}$, the reverse process starts from a fully masked response $\vr_{L}$. At each intermediate step, tokens with high confidence are confirmed, with the others being remasked for the next step prediction. Specifically, the text sequences are divided into several blocks to mimic an auto-regressive process. For MLLMs, the inference is similar to conventional LLMs, following a next-token prediction manner. Here we handle these two types of models individually, leaving their hybrid version, \textit{i.e.}, Multimodal Diffusion Language Model~\cite{yang2025mmada}, or other models with heterogeneous token inputs, for future exploration.

%% file: sections/method.tex
\section{Methodology}
\begin{figure*}[t]
    \centering
    \includegraphics[width=0.9\linewidth]{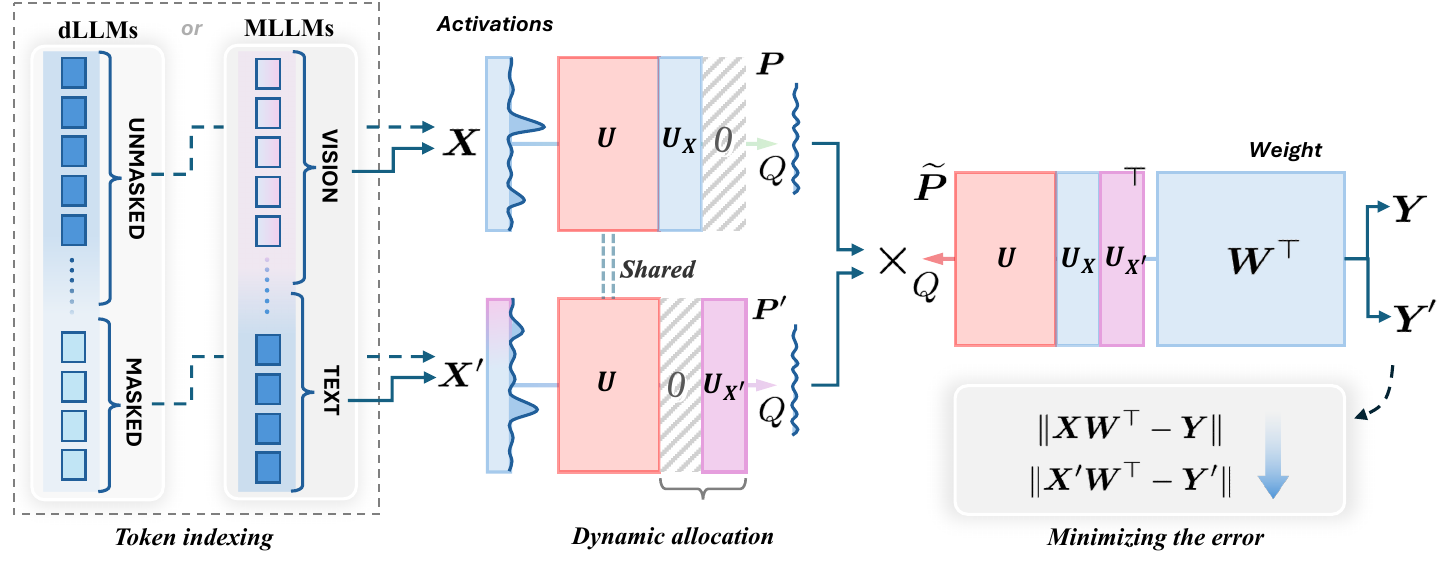}
    \vspace{-2mm}
    \caption{\textbf{The overall framework of FreeAct.} Different tokens are indexed to different transformation matrices according to their unique types (vision-text for MLLMs, and unmasked and masked for dLLMs). Different transformation matrices maintain a shared part together while possessing their own uniqueness, where zeros are filled in on another portion. The weight transformation matrix unites both portions to handle all the different activations. And the quantization error is minimized to optimize the quantization parameter. }
    \label{fig:main}
    \vspace{-2mm}
\end{figure*}

\subsection{Beyond One-to-One}

\textbf{One-to-one transformation.} Transformation matrices are introduced to model the affine connections among channels. Previous work, like per-channel scaling, can be viewed as a special case, where the matrix is diagonal, \textit{i.e.}, $\mY = \mX\mW = (\mX\text{diag}(\vc))\times(\text{diag}(\vc)^{-1}\mW^{\top})$. One matrix for one side, while its inverse is for the other side, to ensure the mathematical equivalence. A more general form can be represented as:
\begin{align}
    \mY = \mX\mP\times \mP^{-1}\mW^{\top}\approx \gQ(\mX\mP)\times \gQ(\mP^{-1}\mW^{\top}),
\end{align}
where $\mP$ is the transformation matrix. FlatQuant \cite{sun2024flatquant} extends the Hadamard transformation, \textit{e.g.}, QuaRot \cite{ashkboos2024quarot}, onto the real field $\mathbb{R}$ (\textit{i.e.}, affine transformation) to flatten the steep and dispersed distribution of activations and weights, thus achieving promising results with acceptable computational overhead. This paradigm establishes \textit{one-to-one correspondence} between activations and weights per layer. $\mP$ is fixed and shared for the dual side to ensure equivalence, \textit{i.e.}, $\mP\mP^{-1}= \mathbf{I}$.

\begin{tcolorbox}
    [colframe=gray!80, colback=gray!10, boxrule=0.2pt, arc=5pt, boxsep=2.5pt, left=2pt, right=2pt, top=2pt, bottom=2pt] \textbf{Dilemmas.} Quantized weights are static during deployment, yet encounter different activation distributions. For instance, the same weights will process activations at different steps and different modal activations in dLLMs and MLLMs, respectively. A \textit{one-fits-all} manner seems limited to handle diverse distributions.
\end{tcolorbox}

To advance the quantization, we explore new possibilities to free the activation side, especially for models that encounter diverse patterns of activations.

\textbf{Freeing the activation side.} For the static weights, we apply dynamic transformations to match the activations. Formally,
\begin{align}
    \left\{ \begin{aligned}\mX\mW^{\top}&= \mX\mP\times\widetilde{\mP}\mW^{\top},\\ \mX'\mW^{\top}&= \mX'\mP'\times\widetilde{\mP}\mW^{\top},\end{aligned} \right. \quad \text{s.t.}\ \mP\not= \mP' \label{eq:problem}
\end{align}
where $\mP$ and $\mP'$ correspond to different activations $\mX$ and $\mX'$. $\widetilde{\mP}$ is determined for weight $\mW$. We are resolving these transformations $\{\mP, \mP',\widetilde{\mP}\}$ to satisfy the above equations.

We transform this question to resolve an equation with a bilinear constraint, \textit{i.e.}, $\mX\mP\widetilde{\mP}\mW^{\top}= \mX\mW^{\top}\to\mX(\mP\widetilde{\mP}- \mI)\mW^{\top}= 0$. The trivial solution is to set the middle term $(\mP\widetilde{\mP}- \mI)$ as 0, leading to $\mP = \widetilde{\mP}^{-1}$. Due to the uniqueness of the matrix inverse, for the common transformation of the weight side, there is no choice for the transformation of different activations.

\begin{prop}
    [Beyond the Inverse] If satisfy the equation of $\mX\mP\widetilde{\mP}\mW^{\top}= \mX\mW^{\top}$, the product $\mP\widetilde{\mP}$ belongs to the set
    \begin{align}
        \Big\{ \mZ - \mP_{\mX}(\mZ- \mI)\mP_{\mW}\;\Big|\; \mZ \in \mathbb{R}^{d \times d}\Big\} \supseteq \{\mI\},
    \end{align}
    which strictly contains the singleton $\{\mathbf{I}\}$. Consequently, the constraint admits a solution space that is strictly larger than the set of exact inverses (for which $\mP\widetilde{\mP}= \mI$). \label{prop:beyond_inverse}
\end{prop}
\vspace{-2mm}
\begin{tcolorbox}
    [colframe=gray!80, colback=gray!10, boxrule=0.2pt, arc=5pt, boxsep=2.5pt, left=2pt, right=2pt, top=2pt, bottom=2pt] \textbf{Remark.} $\mP_{\mX}\in\mathbb{R}^{d \times d}$ and $\mP_{\mW}\in\mathbb{R}^{d \times d}$ serve as the orthogonal projections onto the row spaces of activations and weights, satisfying $\mX\mP_{\mX}= \mX$ and $\mP_{\mW}\mW^{\top}= \mW^{\top}$. A special solution is that $\mP = \bar{\mP}$ and $\widetilde{\mP}= \bar{\mP}^{-1}$, when $\mZ =\mI$. In this case, we assume that the row spaces are of full rank, indicating $\mP_{\mX}= \mP_{\mW}= \mI$. Therefore, we only consider the singleton $\{\mI\}$. In general, activations are \textit{rank-deficient}~\cite{zhang2024magr}, which inspires more solutions to design specific transformations beyond the trivial inverse. Proof is detailed in Appendix~\ref{appendix:beyond_inverse}.
\end{tcolorbox}
\vspace{2mm}

\subsection{FreeAct}
\label{sec:freeact}
To endow activation transformation with more flexibility, we design a new $\mP$ learning method according to the above proposition by enforcing the low rank at the activation side while ensuring the equivalence. The overall framework is shown in Figure~\ref{fig:main}. Different tokens are indexed according to different activation transformation matrices, which are implemented via allocation. Weights are also transformed and multiplied by the activations in quantization, followed by the error minimization to optimize quantization parameters.

\textbf{Token indexing.} Given the token sequence $\vx$, we first split all tokens into different sets. For the sake of generality, we process two types of tokens once. Therefore, one set of indices can be represented as $\gI = \{i\ |\ \sI(\vx_{i}= [\texttt{MASK}]),\ i\le L\}$ per sequence, while another is $\gI' = \{i\ |\ \sI(\vx_{i}\not= [\texttt{MASK}]),\ i\le L\}$. For MLLMs, the indicators are designed to distinguish different modal tokens, \textit{i.e.}, belonging to vision or text. For each linear layer for quantization, assuming the complete input activations $\hat\mX$, we have $\mX = [\hat\mX_{i}]_{i\in\gI}$ and $\mX' = [\hat\mX_{i}]_{i\in\gI'}$. These indexed activations will be transformed and then concatenated following the original order for the fidelity of the sequence. Notably, with this formulation, we do not explicitly differentiate the types of LLMs (\textit{i.e.}, MLLMs or dLLMs) in the following.

\textbf{Dynamic allocation.} For different activations $\{\mX, \mX'\}$, we leverage different transformation matrices, \textit{i.e.}, $\mP$ and $\mP'$, to smooth numerical values \textit{w.r.t.} different channels. Specifically, they are constructed with shared and unique components. 
\begin{align}
&(\mX, \mX') \left\{ \begin{aligned}\mP&= \big[ \mU, \mU_{\mX}, \ \vzero\ \big], \quad \mU\in\sR^{d\times r}, \mU_{\mX}\in\sR^{d\times r_1}, \\ \mP'&= \big[ \mU, \ \vzero\ , \mU_{\mX'}\big],\quad \mU_{\mX'}\in\sR^{d\times r_2}\end{aligned} \right. \nonumber \\&\quad (\mW) \quad \widetilde{\mP}= \big[ \mU, \mU_{\mX}, \mU_{\mX'}\big]^{\top},\quad r+r_{1}+ r_{2}= d
\end{align} 
On one hand, $\mU$ is designed to preserve the shared row space of the activations, maintained by \textit{all} activations. On the other hand, we manually allocate different activations with different portions, where $\mU_{\mX}$ and $\mU_{\mX'}$ are maintained by \textit{distinct} activations. The remaining parts are filled with zeros to avoid any additional information entanglement. This design leads to a well-structured but static weight transformation via $\widetilde{\mP}$. Once $\mX$ is fed into the pre-quantization process, it invokes a low-rank transformation constructed by $\mU$ and $\mU_{\mX}$. Notably, $\mU$, $\mU_{\mX}$, and $\mU_{\mX'}$ represent different subspaces basises to avoid any overlap, and the equivalence
\begin{align}
\mX\mP \times \widetilde{\mP}\mW^{\top}= \mX\mW^{\top},\ \mX'\mP' \times \widetilde{\mP}\mW^{\top}= \mX'\mW^{\top}\label{eq:equality}
\end{align} 
can be guaranteed with analyses in Section~\ref{sec:guarantees_implementation}. Moreover, we provide in-depth analyses on guarantees and implementations from different perspectives for the crux of FreeAct.

\textbf{Minimizing the error.} We optimize these transformations by minimizing the quantization errors under the structure designed above. Following~\cite{sun2024flatquant}, we collect calibration datasets and calculate the differences between the ground truth output and predicted output after quantization. The learning objective is designed per layer:
\begin{align}
    \Ls_{q}= \E_{\vx_0}\Big[ & \| \mX\mW- \gQ(\mX\mP) \gQ(\widetilde{\mP}\mW^{\top})\|_{2}^{2}\nonumber                  \\
    +                         & \| \mX'\mW - \gQ(\mX'\mP') \gQ(\widetilde{\mP}\mW^{\top})\|_{2}^{2}\Big]. \label{eq:loss}
\end{align}
We randomly generate orthogonal matrices for initialization, and optimize them \textit{w.r.t.} different activation types. For dLLMs, we distinguish between unmasked and masked tokens, while for MLLMs, we process vision and text tokens individually to avoid the imbalance training.

\subsection{Guarantees and Implementations}
\label{sec:guarantees_implementation}

To further elucidate the theoretical underpinnings and practical deployment of FreeAct, we provide key analyses and implementation details that ground our design choices. Below, we detail the guarantees of equivalence offered by our method, implementation considerations, and additional enhancements that facilitate superior quantization outcomes.

\begin{thm}
    [Equivalence under projection invariance] Given $\mU, \mU_{\mX}, \mU_{\mX'}$ be semi-orthogonal basis matrices forming an orthogonal decomposition of $\mathbb{R}^{d}$, such that $\mX\mU_{\mX'}= \vzero$ and $\mX'\mU_{\mX}= \vzero$. If we define the orthogonal matrix $\widetilde{\mP}= [\mU, \mU_{\mX}, \mU_{\mX'}]^{\top}$ and the masked transformations $\mP$ and $\mP'$, then for any weight matrix $\mW$, the equalities in Eq.~(\ref{eq:equality}) hold. \label{thm:equivalence}
\end{thm}
\vspace{-2mm}
\begin{tcolorbox}
    [colframe=gray!80, colback=gray!10, boxrule=0.2pt, arc=5pt, boxsep=2.5pt, left=2pt, right=2pt, top=2pt, bottom=2pt] \textbf{Remark.} This theorem substantiates the theoretical validity of our design based on orthogonal decomposition. It guarantees that the linear operations remain equivalent when distinct activation types are projected onto disjoint basis sets. Proof is detailed in Appendix~\ref{appendix:equivalence}. In practice, we approach the zero-projection condition (\textit{e.g.}, $\mX\mU_{\mX'}= \vzero$), rather than enforcing hard constraints, we rely on the minimization of quantization errors to suppress interference between the different token types in Eq.~(\ref{eq:loss}). 
\end{tcolorbox}
\vspace{2mm}

\textbf{Core implementation.} In practice, FreeAct is simple to implement. Since $\mP$ and $\mP'$ are built upon the shared orthogonal matrix $\widetilde\mP$, there is no additional memory cost to store $\mP$ and $\mP'$. To adapt FreeAct, the first step is to identify different activations. We can directly index the tokens by different \texttt{token\_id}, like \texttt{[MASK]} for diffusion models, and \texttt{<IMG>} for MLLMs. Once $\{r,r_{1},r_{2}\}$ are determined, we can slice $\widetilde\mP$ to yield $\mP$ and $\mP'$. We provide the detailed implementation in Appendix~\ref{appendix:implementation} and algorithm flow in Algorithm~\ref{alg:freeact}. The slicing procedure can be executed after the transformation. Here we provide the pseudo-code to adapt FreeAct with only three lines of code added before the quantization, as shown below.

\begin{tcolorbox}
    [colframe=gray!10, colback=gray!10, coltitle=black, top=2pt, bottom=2pt, arc=0pt,left=-5pt, right=0pt] \begin{lstlisting}[language=Python, belowskip=0mm,]
    inp = inp @ P # Naive transformation
    inp[index][..., -r1:] = 0 
    inp[~index][..., -r1 -r2:-r1] = 0
    \end{lstlisting}
\end{tcolorbox}

\input{sections/main_table}

\textbf{Additional enhancement.} We also adopt more techniques to further enhance the quantization performance. \textit{(1) Learnable Clip Threshold.} This threshold manages the reasonable maximum value of activations or weights, enforcing the smoothness by clipping the values beyond the given threshold. We also free the clip threshold of the activation side to offer more flexibility. \textit{(2) Per-Channel Scale.} This technique operates under the observation of the pattern of different channels, magnifying or shrinking values per channel. The scale for activation and weight is entangled to keep the mathematical equivalence, \textit{i.e.}, $\mY = (\mX\cdot\text{diag}(\vc))(\text{diag}(\vc)^{-1}\cdot\mW)$. \textit{(3) Kronecker Product.} To alleviate the computation on matrices $\mP$, we adopt two lightweight matrices to construct a larger one, \textit{i.e.}, $\mP:=\mP_{\text{l}}\otimes\mP_{\text{r}}$, following \cite{chee2023quip, sun2024flatquant}. These techniques are compatible with our proposed method, enabling synergies that boost quantization.

\textbf{Integrations.} We integrate FreeAct with both diffusion and multimodal large language models, which adopt the transformation architecture. We employ the matrix transformation for all linear layers in self-attention and feed-forward modules. For the vision tower in MLLMs, we directly implement FlatQuant due to its homogeneous pattern in activations, where only vision tokens are processed. The specific integration of FreeAct for different models, like LLaDA and QwenVL, is detailed in Appendix~\ref{appendix:implemetation_recipe}.

\textbf{Extensions.} FreeAct provides a sound conceptual and practical design for two different activations. However, the philosophy of FreeAct can be extended to 1) more flexible matrix constructions and 2) more modalities. We discuss the possibilities of these two extensions below and leave them for future exploration. 
(1) Proposition~\ref{prop:beyond_inverse} suggests a larger solution space beyond the sole inverse, showing promise to relax the handcraft to a more flexible matrix learned by better optimization objectives. 
There is a trivial implementation of flex-mode for FreeAct by enlarging the quantization space. In this case, $r_1 = r_2 = d$. It does not assume the explicit shared components \textit{i.e.}, $r=0$ and $r_1+r_2 = 2d$, while increasing the memory cost. 
(2) Moreover, due to the emergence of modalities, like audio~\cite{xu2025qwen2}, FreeAct is capable of involving more modalities if more patterns of activations are identified. The trivial extension is adding a new portion $\mU_{\mX''}\in\mathbb{R}^{d\times r_3}$ that satisfies $r+r_1+r_2+r_3 = d$. In addition, Multimodal Diffusion Large Language Models (MDLLMs)~\cite{yang2025mmada} can be viewed as a hybrid problem. To simplify the analysis, in this paper, we do not implement FreeAct with MDLLMs.

%% file: sections/main_table.tex
\begin{table*}[t]
\centering
\caption{\textbf{Overall performance comparison} of FreeAct against quantization methods across four models. The best result in each case is marked in bold, and the averaged value is under the column ``Avg.''.}
\vspace{-2mm}
\label{tab:overall_performance}
\resizebox{0.9\textwidth}{!}{
\begin{tabular}{l|cccc|cccc} 
\toprule
                    & \multicolumn{4}{c|}{\textbf{dLLMs}}                               & \multicolumn{4}{c}{\textbf{MLLMs}}                                 \\
                                                             & HumanEval      & GSM8K          & Math500        & Avg.           & MMMU           & MMBench        & RealworldQA    & Avg.            \\ 
\cdashline{2-9}
                                                             & \multicolumn{4}{c|}{\textit{LLaDA}}                               & \multicolumn{4}{c}{\textit{Qwen2.5-VL}}                            \\ 
\midrule
\textcolor{lightgray}{16-bit Baseline}                                              & \textcolor{lightgray}{43.90}          & \textcolor{lightgray}{76.80}          & \textcolor{lightgray}{38.40}          & \textcolor{lightgray}{53.03}          & \textcolor{lightgray}{50.44}          & \textcolor{lightgray}{83.25}          & \textcolor{lightgray}{68.50}          & \textcolor{lightgray}{67.40}           \\
\textcolor{lightgray}{RTN (W8A8)}                                                   & \textcolor{lightgray}{39.02}          & \textcolor{lightgray}{77.56}          & \textcolor{lightgray}{37.00}          & \textcolor{lightgray}{51.19}          & \textcolor{lightgray}{49.33}          & \textcolor{lightgray}{82.38}          & \textcolor{lightgray}{68.50}          & \textcolor{lightgray}{66.74}           \\ 
\midrule
RTN (W4A4)                                                   & 00.00          & 00.00          & 00.14          & 00.05          & 26.33          & 00.09          & 00.65          & 09.02           \\ 
\midrule
SmoothQuant                                                  & 00.00          & 00.00          & 00.00          & 00.00          & 24.41          & 00.43          & 01.17          & 08.67           \\
QuaRot                                                       & 28.66          & 65.81          & 24.20          & 39.56          & 43.56          & 72.51          & 56.34          & 57.47           \\
FlatQuant                                                    & 39.63          & 74.37          & 33.60          & 49.20          & 48.89          & 81.61          & \textbf{66.27} & 65.59           \\ 
\midrule
\rowcolor[rgb]{0.949,0.949,0.949} \textbf{FreeAct (ours)}  & \textbf{42.68} & \textbf{75.74} & \textbf{34.60} & \textbf{51.01} & \textbf{50.44} & \textbf{82.22} & 66.01          & \textbf{66.22}  \\ 
\toprule
                                                             & \multicolumn{4}{c|}{\textit{Dream}}                               & \multicolumn{4}{c}{\textit{InternVL2.5}}                           \\ 
\midrule
\textcolor{lightgray}{16-bit Baseline}                                              & \textcolor{lightgray}{56.10}          & \textcolor{lightgray}{68.16}          & \textcolor{lightgray}{42.20}          & \textcolor{lightgray}{55.49}          & \textcolor{lightgray}{53.22}          & \textcolor{lightgray}{84.71}          & \textcolor{lightgray}{68.76}          & \textcolor{lightgray}{68.90}           \\
\textcolor{lightgray}{RTN (W8A8)}                                                   & \textcolor{lightgray}{48.78}          & \textcolor{lightgray}{68.76}          & \textcolor{lightgray}{40.80}          & \textcolor{lightgray}{52.78}          & \textcolor{lightgray}{44.33}          & \textcolor{lightgray}{74.74}          & \textcolor{lightgray}{63.66}          & \textcolor{lightgray}{60.91}           \\ 
\midrule
RTN (W4A4)                                                   & 00.00          & 00.53          & 0.018          & 00.18          & 22.78          & 00.20          & 00.78          & 07.92            \\ 
\midrule
SmoothQuant                                                  & 00.00          & 00.00          & 0.014          & 00.05          & 23.44          & 00.94          & 01.05          & 08.47            \\
QuaRot                                                       & 39.02          & 56.56          & 25.80          & 40.46          & 42.19          & 63.66          & 54.25          & 53.36           \\
FlatQuant                                                    & 46.34          & 63.53          & \textbf{34.60} & 48.16          & 42.11          & 69.84          & 57.65          & 56.53           \\ 
\midrule
\rowcolor[rgb]{0.949,0.949,0.949} \textbf{FreeAct (ours) } & \textbf{50.00} & \textbf{68.84} & 32.60          & \textbf{50.48} & \textbf{42.44} & \textbf{70.62} & \textbf{58.69} & \textbf{57.25}  \\
\bottomrule
\end{tabular}
}
\vspace{-2mm}
\end{table*}

%% file: sections/experiments.tex
\section{Experiments}
In this section, we conduct experiments to address the following research question:
\begin{itemize}[leftmargin=*]
    \item~\textbf{RQ1}: Does our method outperform SOTA quantization methods for dLLMs and MLLMs across diverse benchmarks, to demonstrate its effectiveness?  
    \item~\textbf{RQ2}: Has the rank-deficiency been verified to support FreeAct? And what is the effect of the additional component that enhances FreeAct?
    \item~\textbf{RQ3}: How does FreeAct perform during the training and inference? 
\end{itemize}

\subsection{Experimental Setup}

\textbf{Calibration \& benchmark datasets.} We randomly sample 128 sampels with 2048 tokens each from  WikiText2~\cite{merity2016pointer} to construct the calibration datasets for dLLMs. We randomly select 256 samples with varying lengths from COCO-Cap~\cite{lin2014microsoft} for MLLMs calibration, including one image and one query each. 
To evaluate these two LLM types, we select HumanEval~\cite{chen2021evaluating}, GSM8K~\cite{cobbe2021training} (4-shot), and Math500~\cite{hendrycks2021measuring} (4-shot) for text tasks on dLLMs. For MLLMs, we select vision-relevant tasks, including MMMU~\cite{yue2024mmmu}, MMBench~\cite{liu2024mmbench}, and RealworldQA\footnote{\url{https://huggingface.co/datasets/xai-org/RealworldQA}.}.
We also report the average value as an overall performance indicator. The dataset descriptions are detailed in Appendix~\ref{appendix:cal_dataset} and~\ref{appendix:eval_dataset}.

\textbf{Base models \& baselines.} 
We select two models for each LLM paradigm, including LLaDA~\cite{nie2025large}, Dream~\cite{ye2025dream}, Qwen2.5VL~\cite{bai2025qwen2}, and InternVL2.5~\cite{chen2024internvl}. Built upon these models, we test the quantization quality of FreeAct against RTN, SmoothQuant~\cite{xiao2023smoothquant}, QuaRot~\cite{ashkboos2024quarot}, and FlatQuant~\cite{sun2024flatquant} under the W4A4 setting. In particular, we also report the results of the 16-bit baseline and RTN in W8A8 for a comprehensive comparison. The details of base models and baselines are shown in Appendix~\ref{appendix:base_models}.

\textbf{Implementation details.} With collected calibration datasets, we tune the quantization parameters, including the transformation matrices and clip thresholds, with the AdamW optimizer. The learning rate is set to 1e-3, and the number of training epochs is 15 with a batch size of 4. These settings are kept the same for all the methods. For FreeAct, we set the ratios $r_1=r_2=d/k$, and set $k=32$ by default for simplicity. We also select the front and back layers for FreeAct (see the fully applied version in Section~\ref{sec:further_analysis}). We provide more implementation recipes in Appendix~\ref{appendix:implemetation_recipe}.

\subsection{Overall Performance (RQ1)}

We report the results across different models, datasets, and quantization methods in Table~\ref{tab:overall_performance}. Analyzing the comparative results offers the following three key observations:

\ding{172} \textbf{FreeAct outperforms all quantization baselines, achieving SOTA performance {in most scenarios}.} Across both dLLM benchmarks (\textit{e.g.}, HumanEval and GSM8K) and MLLM benchmarks (\textit{e.g.,} MMMU and MMBench), FreeAct surpasses existing W4A4 methods in most cases. Notably, it demonstrates superior accuracy retention compared to strong baselines like FlatQuant, achieving up to 5.3\% improvement. In several tasks, FreeAct recovers the performance to a level comparable with W8A8 RTN and the 16-bit base models, like in Dream and Qwen2.5VL, highlighting its efficacy in mitigating quantization noise in ultra-low bit-width settings.

\ding{173} \textbf{The transformation method is better than simple channel-wise quantization (SmoothQuant).} The results indicate a clear performance gap between transformation-based methods (FreeAct, FlatQuant, QuaRot) and simple scaling methods like SmoothQuant. While SmoothQuant mitigates outliers via channel-wise scaling, it struggles to fully suppress quantization errors in the presence of complex activation patterns under the low-bit scenario. FreeAct, by employing a more sophisticated transformation strategy, effectively flattens the activation distribution, making it more amenable to W4A4 quantization.

\ding{174} \textbf{The larger flexibility on the transformation matrices provides the possibility to achieve better quantization quality.} Comparing FreeAct against FlatQuant and QuaRot reveals that increased flexibility within the transformation matrix is crucial. FlatQuant eases the Hadamard transformations in QuaRot and employs an affine flexibility, yet it remains tethered to a one-to-one mapping. FreeAct evolves this by breaking this constraint, allowing for a many-to-one transformation. This increased flexibility allows FreeAct to better accommodate diverse activation distributions, resulting in superior quantization quality.

\subsection{Ablation Study (RQ2)}

\begin{figure}
    \centering
    \includegraphics[width=\linewidth]{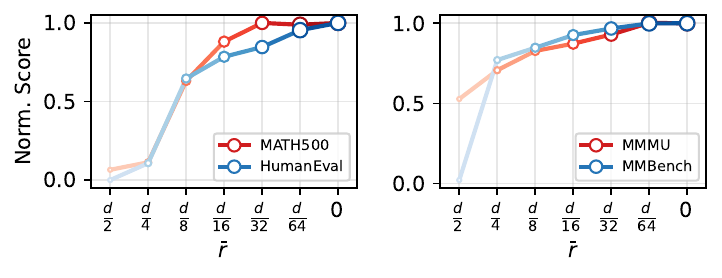}
    \vspace{-9mm}
    \caption{Performance comparison under different rank-deficient settings, tested on LLaDA (\textit{left}) and Qwen2.5VL (\textit{right}).}
    \vspace{-5mm}
    \label{fig:rank_ablation}
\end{figure}

\textbf{Rank-deficient verification.} We further investigate the necessity of full-rank transformations by implementing low-rank one-to-one transformation variants. Comparisons on two different models are reported in Figure~\ref{fig:rank_ablation}, where $\bar{r}$ denotes the number of removed dimensions, and 0 represents the full-rank transformation. 
We can observe that a large portion of the removal is incapable, while removing $d/32$ or $d/64$ can approach the upper bound performance. These results confirm our motivation for the rank-deficiency and the feasibility of developing FreeAct. 

\begin{table}
\centering
\caption{Performance comparison with FreeAct without a learnable clip threshold (L.C.).}
\vspace{-2mm}
\label{tab:sole_transform}
\resizebox{\linewidth}{!}{
\begin{tabular}{l|cc!{\vrule width \lightrulewidth}cc} 
\toprule
                 & \multicolumn{2}{c!{\vrule width \lightrulewidth}}{\textit{LLaDA}}                              & \multicolumn{2}{c}{\textit{Dream}}                                                              \\
                 & Math                                      & HumanEval                                 & Math                                      & HumanEval                                  \\ 
\midrule
FreeAct$_\text{w/o L.C.}$ & 73.09                                     & 34.15                                     & 55.65                                     & 34.14                                      \\
FreeAct          & \textbf{\textbf{\textbf{\textbf{75.74}}}} & \textbf{\textbf{\textbf{\textbf{34.60}}}} & \textbf{\textbf{\textbf{\textbf{68.84}}}} & \textbf{\textbf{\textbf{\textbf{50.00}}}}  \\ 
\midrule
                 & \multicolumn{2}{c!{\vrule width \lightrulewidth}}{\textit{Qwen2.5VL}}                          & \multicolumn{2}{c}{\textit{InternVL2.5}}                                                        \\
                 & MMMU                                      & MMBench                                   & MMMU                                      & MMBench                                    \\ 
\midrule
FreeAct$_\text{w/o L.C.}$ & 46.44                                     & 79.30                                     & 40.22                                     & 65.12                                      \\
FreeAct          & \textbf{\textbf{50.44}}                   & \textbf{\textbf{82.22}}                   & \textbf{\textbf{42.44}}                   & \textbf{\textbf{70.62}}                    \\
\bottomrule
\end{tabular}
}
\vspace{-4mm}
\end{table}

\textbf{The effect of learnable clip threshold.} We isolate the contribution of the transformation matrix from clipping threshold tuning to evaluate a ``clean'' version of FreeAct (w/o L.C.) where only the transformation matrix is learned. The results indicate that the optimized transformation matrix is the primary driver of FreeAct's success in most cases, like LLaDA. The clip threshold also contributes to capturing the dynamics, especially in Dream, where performance drops with a noticeable margin, suggesting a synergy produced by both the transformation and clip threshold. 

\subsection{Further Analysis (RQ3)}
\label{sec:further_analysis}
\textbf{The loss of quantization errors.} We analyze the training dynamics of FreeAct by plotting the quantization error over time as shown in Figure~\ref{fig:loss_combine}\footnote{The complete curves for four models are in Appendix~\ref{appendix:activation_vis}.}. We observe a consistent decrease in Mean Squared Error (MSE) for both types of activations. We also observe the different quantization error ranges between different types ($\mX>\mX'$), which suggests the need for a dynamic process. The convergence behavior confirms that FreeAct effectively optimizes the quantization parameters during the calibration phase, progressively minimizing the information loss caused by low-bit quantization.

\begin{figure}
    \centering
    \includegraphics[width=\linewidth]{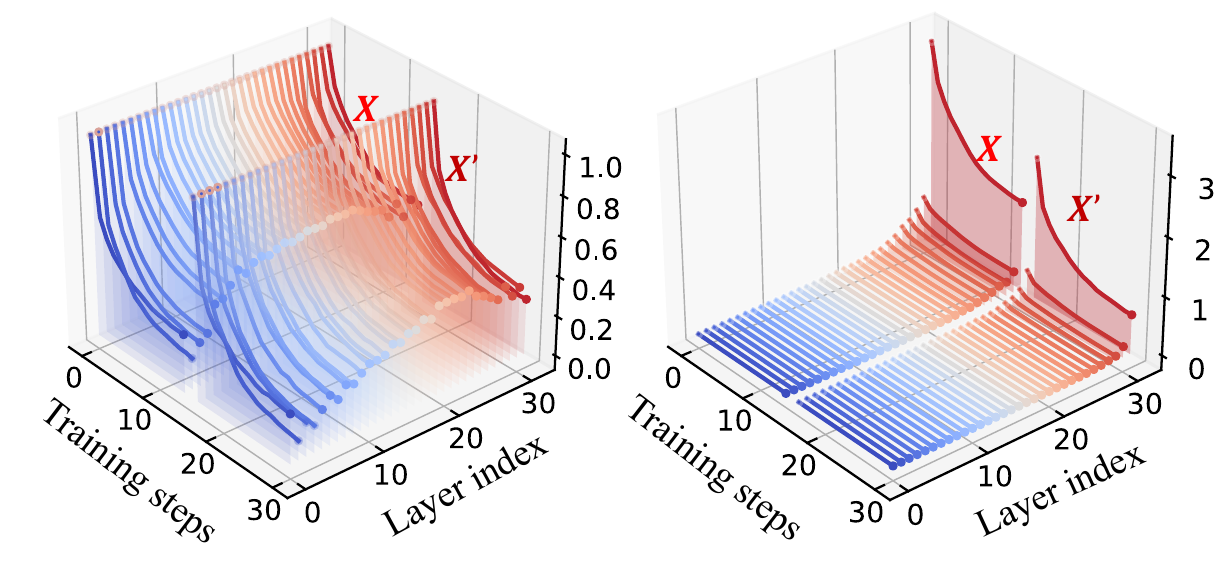}
    \vspace{-8mm}
    \caption{Quantization errors optimized along with the training steps across different layers. $\mX$ and $\mX'$ denote different activation types, \textit{i.e.}, masked and unmasked token. The values in the left figure are normalized per layer.}
    \label{fig:loss_combine}
    \vspace{-4mm}
\end{figure}

\begin{figure}[t]
    \centering
    \vspace{-2mm}
    \includegraphics[width=\linewidth]{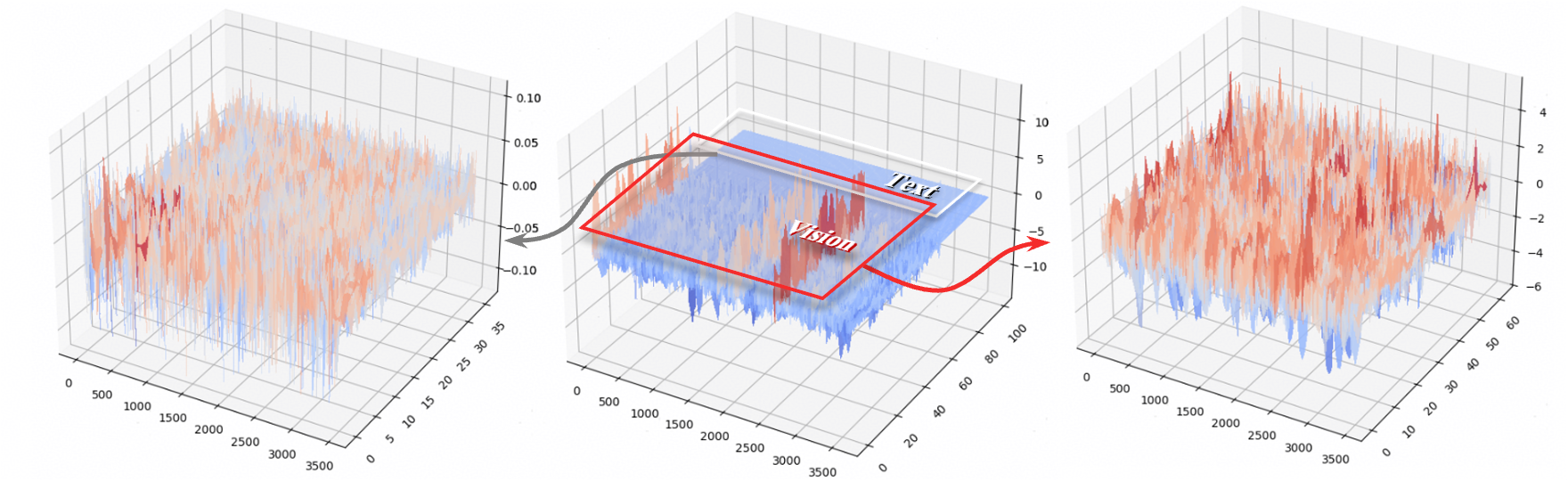}
    \caption{Activation distributions after the transformation.}
    \vspace{-2mm}
    \label{fig:act_vis}
\end{figure}

\textbf{Activation visualization after transformation.}
We illustrate the smoothed activation ranges after applying the FreeAct transformation in Figure~\ref{fig:act_vis}. Compared to the original activations. Applying the transformation successfully redistributes the activation values into a narrower, more uniform range. This smoothing effect ensures that the limited bins available in W4A4 quantization are utilized more efficiently, directly correlating with the reduced quantization error observed in our quantitative analysis.

\textbf{Applying all layers with the fixed ratio.}
FreeAct provides more flexibility on the transformation matrix for each layer. However, the fixed mask ratio for all layers is not optimal. We show the comparison results with this fully applied version, FreeAct$_\text{full}$, as shown in Figure~\ref{fig:compare}. Simply enforcing the activation isolation in FreeAct to every layer would lead to a worse performance in some cases, like on LLaDA or Dream, or perform well on Qwen2.5VL. This suggests the promise of FreeAct, where sophisticated strategies would further boost FreeAct. 

\begin{figure}
    \centering
    \vspace{-3mm}
    \includegraphics[width=\linewidth]{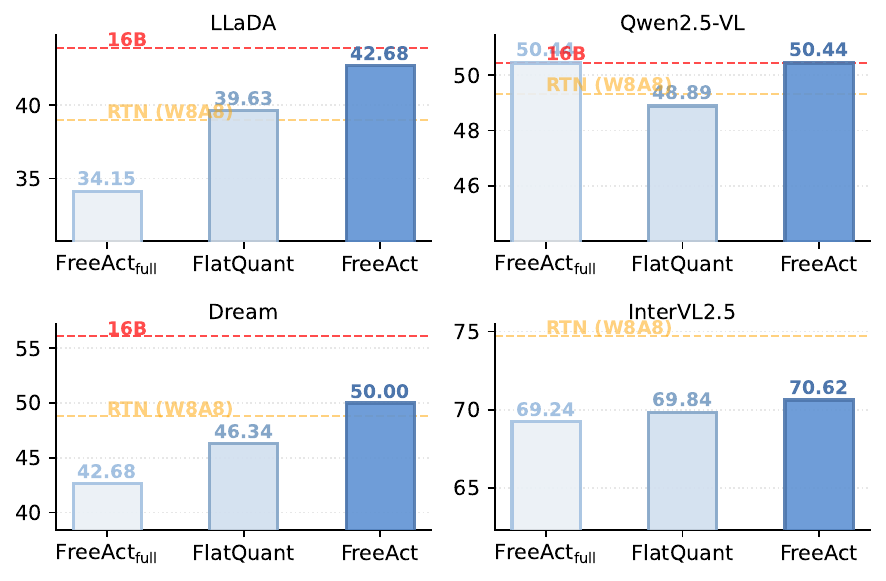}
    \vspace{-6mm}
    \caption{Performance comparison across FlatQuant and fully applied FreeAct.}
    \label{fig:compare}
    \vspace{-6mm}
\end{figure}

%% file: sections/discussions.tex
\section{Discussions and Outlook}
In this paper, we introduced FreeAct, a novel post-training quantization framework that breaks the conventional one-to-one transformation constraint prevalent in current LLM quantization methods. By identifying the limitations of static transformations in handling the dynamic activation patterns of advanced models, like dLLMs and MLLMs, we proposed a paradigm shift towards flexible activation transformations.
Theoretical analysis of the rank-deficient nature of activations allowed us to decouple the activation transformation from the weight transformation. FreeAct successfully allocates distinct transformation matrices to different token types while maintaining a unified, static transformation for weights. Our extensive experiments demonstrate that FreeAct effectively mitigates quantization errors arising from dynamic disparities. FreeAct establishes a robust foundation for dynamic transformation in quantization, paving several promising avenues for future explorations: 1) extending this framework to more modalities or hybrid model architectures, 2) investigating hardware-kernel co-design for full exploitation, and 3) generalizing to automatically token identification rather than indexing on predefined types.

%% file: sections/appendix.tex
\section{Theoretical Analysis}

\subsection{Proof of Proposition~\ref{prop:beyond_inverse}}
\label{appendix:beyond_inverse}

\textit{Recall.} \textit{To satisfy $\mX\mP\widetilde{\mP}\mW^{\top}= \mX\mW^{\top}$, the product $\mP\widetilde{\mP}$ belongs to the set
\[
    \Big\{ \mZ - \mP_{\mX}(\mZ- \mI)\mP_{\mW}\;\Big|\; \mZ \in \mathbb{R}^{d \times d}\Big\} \supseteq \{\mI\},
\]
which strictly contains the singleton $\{\mathbf{I}\}$. Consequently, the constraint admits a solution space that is strictly larger than the set of exact inverses (for which $\mP\widetilde{\mP}= \mI$). }
\begin{proof}
    \begin{align}
        \mX\mP\widetilde{\mP}\mW^{\top}= \mX\mW^{\top}\implies \mX\mP\widetilde{\mP}\mW^{\top}- \mX\mW^{\top} & = 0 \\
        \mX(\mP\widetilde{\mP}- \mI)\mW^{\top}                                                                & = 0
    \end{align}
    According to the general solution for $\mA\mX\mB=\mC$, see Lemma~\ref{lem:AXB=C}, we have:
    \begin{align}
        \mX(\mP\widetilde{\mP}- \mI)\mW^{\top}= 0\implies \mP\widetilde\mP -\mI & = \mZ -\mA^{\dagger}\mA\mZ\mB\mB^{\dagger}                                 \\
        \mP\widetilde\mP                                                        & = \mZ -\mA^{\dagger}\mA\mZ\mB\mB^{\dagger}+ \mI                           &  \\
                                                                                & = \mZ - \mA^{\dagger}\mA (\mZ-\mI) \mB\mB^{\dagger}\quad ( \mZ:=\mZ - \mI)
    \end{align}
\end{proof}
\begin{lem}
    [Solving $\mA\mX\mB=\mC$~\cite{penrose1955generalized}] A necessary and sufficient condition for the equation $\mA\mX\mB= \mC$ to have a solution is $\mA\mA^{\dagger}\mC\mB^{\dagger}\mB= \mC$, in which case the general solution is
    \begin{align}
        \mX= \mA^{\dagger}\mC\mB^{\dagger}+ \mZ -\mA^{\dagger}\mA\mZ\mB\mB^{\dagger},
    \end{align}
    where $\mZ$ is arbitrary. \label{lem:AXB=C}
\end{lem}
\begin{proof}
    \begin{align}
        \mA\mX\mB=\mC \implies \mC = \mA\mX\mB = \mA\mA^{\dagger}\mA \mX\mB\mB^{\dagger}\mB = \mA(\mA^{\dagger} \mC\mB^{\dagger})\mB.
    \end{align}
    $\mA^{\dagger}\mC\mB^{\dagger}$ is the particular solution of $\mA\mX\mB=\mC$. For the general solution, we are solving $\mA\mX\mB = 0$.
    \begin{align}
        \mX = \mZ - \mA^{\dagger}\mA\mZ\mB\mB^{\dagger}\implies \mA\mX\mB & = \mA \left(\mZ - \mA^{\dagger}\mA\mZ\mB\mB^{\dagger}\right)\mB \\
                                                                          & = \mA\mZ\mB - \mA\mA^{\dagger}\mA \mZ \mB \mB^{\dagger}\mB      \\
                                                                          & = \mA\mZ\mB - \mA\mZ\mB = 0
    \end{align}
    If $\mX$ is one solution of $\mA\mX\mB = 0$, then
    \begin{align}
        \mA\mX\mB = 0 \implies \mX = \mX - 0 = \mX - \mA^{\dagger}(\mA\mX\mB)\mB^{\dagger}
    \end{align}
\end{proof}

\subsection{Proof of Theorem~\ref{thm:equivalence}}
\label{appendix:equivalence}

\textit{Recall.} \textit{Given $\mU, \mU_{\mX}, \mU_{\mX'}$ be semi-orthogonal basis matrices forming an orthogonal decomposition of $\mathbb{R}^{d}$, such that $\mX\mU_{\mX'}= \vzero$ and $\mX'\mU_{\mX}= \vzero$. If we define the orthogonal matrix $\widetilde{\mP}= [\mU, \mU_{\mX}, \mU_{\mX'}]^{\top}$ and the masked transformations $\mP = [\mU, \mU_{\mX}, \vzero]$ and $\mP' = [\mU, \vzero, \mU_{\mX'}]$, then for any weight matrix $\mW$, the equalities
\[
    \mX\mP\widetilde{\mP}\mW^{\top}= \mX\mW^{\top}\quad \text{and}\quad \mX'\mP'\widetilde{\mP}\mW^{\top}= \mX'\mW^{\top}
\]hold.}

\begin{proof}
    Give the identity property of the orthogonal matrix $\widetilde{\mP}$, we have:
    \begin{align}
        \widetilde{\mP}\widetilde{\mP}^{\top} & = \big[ \mU, \mU_{\mX}, \mU_{\mX'}\big] \begin{bmatrix}\mU^{\top}\\ \mU_{\mX}^{\top}\\ \mU_{\mX'}^{\top}\end{bmatrix} \\
                                              & = \mU\mU^{\top}+ \mU_{\mX}\mU_{\mX}^{\top}+ \mU_{\mX'}\mU_{\mX'}^{\top}                                               \\
                                              & = \mI_{d}
    \end{align}
    Since the rank-deficiency of $\mX$, we have $\text{rank}(\mX)=r+r_{1}<d$. We can find a matric $\mU_{\mX'}\in\sR^{d\times r_2}, r_{2}= d - (r+r_{1})$ satisfy $\mX\mU_{\mX'}\mU_{\mX'}^{\top}= \vzero$, where $\mU_{\mX'}\mU_{\mX'}^{\top}$ projects $\mX$ to the its kernel space. Therefore,
    \begin{align}
        \mX\mP \widetilde{\mP}\mW^{\top} & = \mX\big[ \mU, \mU_{\mX}, \vzero \big] \begin{bmatrix}\mU^{\top}\\ \mU_{\mX}^{\top}\\ \mU_{\mX'}^{\top}\end{bmatrix}\mW^{\top} \\
                                         & = \mX(\mU\mU^{\top}+ \mU_{\mX}\mU_{\mX}^{\top})\mW^{\top}                                                                       \\
                                         & =\mX(\mI_{d}- \mU_{\mX'}\mU_{\mX'}^{\top})\mW^{\top}\quad \text{($\widetilde{\mP}\widetilde{\mP}^{\top}= \mI_{d}$)}             \\
                                         & = \mX\mW^{\top}- \underbrace{\mX\mU_{\mX'}\mU_{\mX'}^{\top}}_{\vzero}\mW^{\top}= \mX\mW^{\top}
    \end{align}
    This also holds for $mX'$, for its rank-deficiency $\text{rank}(\mX')=r+r_{2}<d$. Therefore:
    \begin{align}
        \mX'\mP' \widetilde{\mP}\mW^{\top} & = \mX'(\mI_{d}- \mU_{\mX}\mU_{\mX}^{\top})\mW^{\top}= \mX'\mW^{\top}- \mX'\mU_{\mX}\mU_{\mX}^{\top}\mW^{\top}= \mX'\mW^{\top}.
    \end{align}
\end{proof}

\input{sections/related_work}

\section{Implementation Details}
\label{appendix:implementation}

\subsection{Calibration Datasets}
\label{appendix:cal_dataset}
\begin{itemize}
    \item \textbf{WikiText2}~\cite{merity2016pointer}\footnote{\url{https://huggingface.co/datasets/Salesforce/wikitext}} is a language modeling dataset comprising approximately 2 million tokens on Wikipedia. It maintains a vocabulary size of 33,278 words and retains the original document structure, including full punctuation and capitalization. We directly concatenate different sequences and split them into 128 samples with 2048 tokens for quantization calibration of dLLMs. 
    \item \textbf{COCO-Cap}~\cite{lin2014microsoft}\footnote{\url{https://huggingface.co/datasets/lmms-lab/COCO-Caption}} is a large-scale multimodal collection consisting of approximately 120,000 images derived from the MS COCO dataset, with each image annotated with five independent human-generated captions. We randomly select 256 samples from its validation split, serving as the calibration dataset for MLLM quantization.
\end{itemize}

\subsection{Evaluation Benchmarks}
\label{appendix:eval_dataset}

\begin{itemize} 
\item \textbf{GSM8K}~\cite{cobbe2021training} is a dataset of approximately 8,500 high-quality, linguistically diverse grade school math word problems created by human writers. We employ its test split, which contains 1,319 examples in total. 
\item \textbf{Human-Eval}~\cite{chen2021evaluating} is a code generation benchmark released by OpenAI, consisting of 164 hand-written programming problems in Python. Each problem includes a function signature, a docstring, and a set of hidden unit tests. 
\item \textbf{Math500}~\cite{hendrycks2021measuring} is a curated subset of the larger MATH dataset, containing 500 challenging mathematics problems covering diverse subjects such as algebra, geometry, and calculus. It serves as a robust benchmark for evaluating the advanced mathematical problem-solving capabilities of LLMs. 
\end{itemize}

\begin{itemize} 
\item \textbf{MMMU}~\cite{yue2024mmmu} is a large-scale benchmark designed to evaluate expert-level multimodal reasoning. It comprises approximately 11,500 questions derived from college exams, quizzes, and textbooks across six core disciplines and 30 distinct subjects. We use its validation split, comprising 900 multiple-choice questions. 
\item \textbf{MMBench}~\cite{liu2024mmbench} is a comprehensive evaluation pipeline designed to assess the fine-grained perception and reasoning abilities of MLLMs. It employs a circular evaluation strategy with multiple-choice questions to ensure robust performance assessment across various perceptual and cognitive dimensions. 
\item \textbf{RealworldQA} is a benchmark designed to test spatial and physical reasoning in real-world scenarios. It features high-quality images taken from vehicles and egocentric views, challenging models to answer questions about object relations and environmental context in unconstrained, realistic settings. 
\end{itemize}

The evaluation of dLLMs is evaluated by \texttt{LM\_Eval}\footnote{\url{https://github.com/EleutherAI/lm-evaluation-harness}} and further accelerated with \texttt{Fast-dLLM}\footnote{\url{https://github.com/NVlabs/Fast-dLLM}}. The evaluation of MLLMs is evaluated by \texttt{lmms-eval}\footnote{\url{https://github.com/EvolvingLMMs-Lab/lmms-eval}}. 

\subsection{Base Models}
\label{appendix:base_models}

\begin{itemize}
    \item \textbf{LLaDA}~\cite{nie2025large} is a generative Large Language Model built on a diffusion architecture by employing a masked diffusion strategy. It iteratively denoises the entire sequence in parallel during both training and inference, demonstrating that diffusion models can achieve generation quality and in-context learning capabilities comparable to autoregressive baselines. We use its instruction version, \texttt{LLaDA-8B-Instruct}, for our evaluation.
    \item \textbf{Dream}~\cite{ye2025dream} is a diffusion large language model that employs discrete diffusion modeling to optimize the sequences in parallel. Dream is initialized with weights from Qwen2.5 7B and is continually pre-trained to adapt the diffusion architecture. We use its instruction version, \texttt{Dream-v0-Instruct-7B}, for our evaluation.
    
    \item \textbf{QwenVL}~\cite{bai2025qwen2} is a vision-language model built upon the Qwen foundation, designed to handle high-resolution image inputs and fine-grained visual grounding tasks. It incorporates position-aware visual adapters that allow it to perform complex tasks, establishing itself as a robust open-source benchmark for general-purpose multimodal perception and interaction. We use \texttt{Qwen2.5-VL-7B-Instruct} for our evaluation.
    
    \item \textbf{InternVL}~\cite{chen2024internvl} employs a progressive scaling strategy that aligns a 6B vision encoder with smaller LLMs before seamlessly transferring weights to larger models to cut computational costs. The well-designed data pipelines and the advanced reweighting technique that balance gradients across diverse response lengths and enhance the visual understanding. Here we use \texttt{InternVL2\_5-8B} for our evaluation.
\end{itemize}

\subsection{Baselines}

\begin{itemize}
    \item \textbf{RTN} is the straightforward quantization strategy that maps floating-point values to the nearest integer representation without additional optimization or calibration.
    \item \textbf{SmoothQuant}~\cite{xiao2023smoothquant} uses a diagnoal scales to smooth activation outliers by migrating the quantization difficulty from activations to weights.
    \item \textbf{QuaRot}~\cite{ashkboos2024quarot} uses randomized Hadamard transforms to rotate weights and activations into a space where outliers are suppressed, enabling outlier-free 4-bit quantization.
    \item \textbf{FlatQuant}~\cite{sun2024flatquant} is designed to improve low-bit quantization by flattening the activation distributions using similar rotation techniques, specifically optimized for efficient deployment on hardware.

\end{itemize}

\subsection{Implementation Recipe} 
\label{appendix:implemetation_recipe}

In the following, we first review the one-to-one transformation recipe, exemplified by FlatQuant. Then we provide a comprehensive implementation of FreeAct.

\textbf{One-to-one recipe. } 
Give a linear layer equipped with weight $\mW\in\mathbb{R}^{d'\times d}$, the input activation is denoted as$\mX\in\mathbb{R}^{n\times d}$, satisfying $\mY = \mX\mW^{\top}\in\mathbb{R}^{n\times d'}$. For the weight side, $\widetilde\mW \gets \mP_{l}^{-1}\widetilde{\times\mW/\diag{\vc}}\times(\mP_{r}^{-1})^{\top}$. And $\mP_{l}:= \mU_{l}\times\diag{\boldsymbol{\vs}_l}\times\mV_{l}^{\top}$ is constructed via SVD to guarantee the invertibility of $\mP$ . $\mU$ and $\mV$ are both orthogonal matrices, and $\vs$ is the vector containing singular values. The transformed weight$\widetilde\mW$is then clipped, operating  $\text{clamp}(\widetilde\mW, s^{\min}_{\widetilde\mW}, s^{\max}_{\widetilde\mW})$, where  $s^{\min}_{\widetilde\mW}:= (\min{\widetilde\mW})\cdot\sigma(\alpha^{\min}_{\mW})$ and $s^{\max}_{\widetilde\mW}:= (\max{\widetilde\mW})\cdot\sigma(\alpha^{\max}_{\mW})$. $\alpha$ is the learnable clip threshold. For the activation side, similarly, $\widetilde\mX \gets \mP_{l}^{\top}\widetilde{\times\mX\cdot\diag{\vc}}\times\mP_{r}$, followed by $\text{clamp}(\widetilde\mX, s^{\min}_{\widetilde\mX}, s^{\max}_{\widetilde\mX})$. After the transformation, we adopt the quantizer from GPTQ. Specifically, the quantization scale for weight is estimated by the minimum and maximum values per channel without MSE search. For the activation, the scale is estimated per token.

\textit{Equivalence.} The equivalence of one-to-one transformation methods can be simply verified by the following equations: $\widetilde{\mX}\widetilde{\mW}^{\top}= (\mP_{l}^{\top}\widetilde{\times\mX\cdot\diag{\vc}}\times\mP_{r})(\mP_{l}^{-1}\widetilde{\times\mW/\diag{\vc}}\times(\mP_{r}^{-1})^{\top})^{\top}= \mX\mP\times\mP^{-1}\mW^{\top}= \mX\mW^{\top}.$

\textbf{More-to-one recipe (FreeAct).} The processing for one linear layer mostly follows the one-to-one recipe. Let $\mX := \widetilde{\mX\cdot\diag{\vc}}$ and $\mW := \widetilde{\mW/\diag{\vc}}$, we assign the same projection matrix for different activations, \textit{i.e.}, $\mX$ and $\mX'$, yet masking its different portions. For instance, $\widetilde{\mX} := (\mP_{l}^{\top}\widetilde{\times\mX\cdot\diag{\vc}}\times\mP_{r}) \odot \vm$, where $\vm \in \{0,1\}^{d}$ is a binary vector and $\odot$ represents the element wise multiplication. $\vm$ controls the utilization of different components for different activation types. For other parts, we keep the same as the one-to-one recipe. The \textit{equivalence} of the more-to-one recipe is proved in Appendix~\ref{appendix:equivalence}.

\textbf{Implementation for dLLMs.} 
\label{appendix:implementation_dllm}
We utilize the LLaDA and Dream implementations from Fast-dLLM to facilitate inference via its optimized caching techniques. The maximum number of generated tokens is set to 256 using 8 time blocks. For calibration, we construct 128 examples containing 2,048 text tokens followed by 256 mask tokens. The activation indices are determined by the \texttt{[MASK]} token ID prior to the next denoising step; consequently, we do not need to explicitly simulate distinct diffusion steps. To accelerate the calibration process, we save all input and output hidden states per layer and train these layers in parallel.

\textbf{Implementation for MLLMs.} 
\label{appendix:implementation_mllm}
We adopt tailored strategies for the various components within the MLLMs. For the LM backbone, we employ FreeAct as it effectively handles diverse activation inputs, such as vision and text. Given that the sequence length for each calibration example is relatively short, we use 256 examples for the calibration process. Activation indices are determined by the \texttt{[IMG]} token ID during prefilling. We apply FlatQuant to the vision model due to its homogeneous tokens and adopt RTN for the MLP projector. Similar to the dLLM setup, we save all input and output hidden states for both the vision encoder and LM decoder to facilitate parallelized training and faster calibration.

\subsection{Hyperparameter Settings}
\label{appendix:hyperparameters}

We employ the AdamW optimizer~\cite{loshchilov2017decoupled} to optimize the quantization parameters, including the transformation matrices and the clip threshold, while freezing the model parameters. We set the learning rate as $1\times10^{-3}$ and train the parameters on calibrated datasets for 15 epochs. The number of samples is 128 for dLLMs and 256 for MLLMs. The transformation matrix is initialized as a random orthogonal matrix and is restricted to being orthogonal during the optimization. The clip threshold is initialized as 4.0. We keep these hyperparameters the same for both baselines and FreeAct. The key hyperparameters for FreeAct are $\{r,r_1, r_2\}$. To simplify the implementation, we set $r_1 = r_2 = d/k$. We set $k=32$ by default.

\subsection{Detailed Algorithm Flow}
\label{appendix:algorithm}

In this section, we provide the training algorithm flow of FreeAct in Algorithm~\ref{alg:freeact}. 
\begin{algorithm}[h]
\caption{FreeAct  Algorithm Flow}
\label{alg:freeact}
\begin{algorithmic}[1]
\REQUIRE 
    Calibration dataset $\mathcal{D} = \{\mX_k\}_{k=1}^N$ where $\mX_k \in \sR^{L \times d}$, pre-trained weights $\mW \in \sR^{d \times d_{out}}$, $r, r_1, r_2$.
\ENSURE 
    \STATE Optimized orthogonal matrix $\widetilde{\mP} \in \sR^{d \times d}$. Other quantization parameters (scales, clip threshold).

\STATE \textbf{Initialization:}
\STATE Initialize $\widetilde{\mP}\in{d \times d}$ as a random orthogonal matrix.

\WHILE{not converged}
    \FOR{each batch of input activations $\hat{\mX} \in \sR^{B \times L \times d}$ in $\mathcal{D}$}
        \STATE \textbf{\ding{172} Token Indexing.}\\
        \hskip 2em Identify indices based on token types (e.g., Vision/Text or Unmasked/Masked):\\
        \hskip 2em $\mathcal{G} := \{i \mid \mathcal{I}(\hat{\mX}_{:,i,:}) \text{ is Type A}\}$
        \hskip 2em $\mathcal{G}' := \{i \mid \mathcal{I}(\hat{\mX}_{:,i,:}) \text{ is Type B}\}$
        \STATE \textbf{\ding{173} Dynamic Allocation.}\\
        \hskip 2em $\widetilde{\mP} \leftarrow [\mU, \mU_{\mX}, \mU_{\mX'}]^\top$ \hskip 2em $\rhd$ $\mU \in \sR^{d \times r}$
        \hskip 2em $\mU_{\mX} \in \sR^{d \times r_1}$ 
        \hskip 2em $\mU_{\mX'} \in \sR^{d \times r_2}$ \\
        \hskip 2em $\mP \leftarrow [\mU, \mU_{\mX}, \vzero]\ \in \mathbb{R}^{d \times d}$, where $\vzero \in \sR^{d \times r_2}$\\
        \hskip 2em $\mP' \leftarrow [\mU, \vzero, \mU_{\mX'}]\in \mathbb{R}^{d \times d}$, where $\vzero \in \sR^{d \times r_1}$

        \STATE \textbf{\ding{174} Grouping}\\
        \hskip 2em $\mX \leftarrow \hat{\mX}_{:,\mathcal{G},:} \in \mathbb{R}^{B \times |\mathcal{G}| \times d}$\\
        \hskip 2em $\mX' \leftarrow \hat{\mX}_{:,\mathcal{G}',:} \in \mathbb{R}^{B \times |\mathcal{G}'| \times d}$

        \STATE \textbf{\ding{175} Transformation \& Quantization} \\
        \hskip 2em  $\widetilde{\mW} \leftarrow \widetilde{\mP} \mW$
        \quad  $\widetilde{\mX} \leftarrow \mX \mP'$
        \quad  $\widetilde{\mX'} \leftarrow \mX' \mP'$
        
        \hskip 2em $\mY_{q} \leftarrow \gQ(\widetilde{\mX}) \gQ(\widetilde{\mW})$ \quad $\mY'_{q} \leftarrow \gQ(\widetilde{\mX'}) \gQ(\widetilde{\mW})$

        \STATE \textbf{\ding{176} Optimization}\\
        \hskip 2em $\mathcal{L} \leftarrow \|\mX\mW - \mY_{q}\|_2^2 + \|\mX'\mW - \mY'_{q}\|_2^2$  \quad(Eq.~(\ref{eq:loss})) \\
        \hskip 2em  Update $\widetilde{\mP}$ and quantization params to minimize $\mathcal{L}$.
    \ENDFOR
\ENDWHILE
\end{algorithmic}
\end{algorithm}

\section{Additional Results}
\label{appendix:additional_res}

In this section, we provide the full illustrations of activations per layer across different models.

\begin{figure*}
    \centering
    \includegraphics[width=\linewidth]{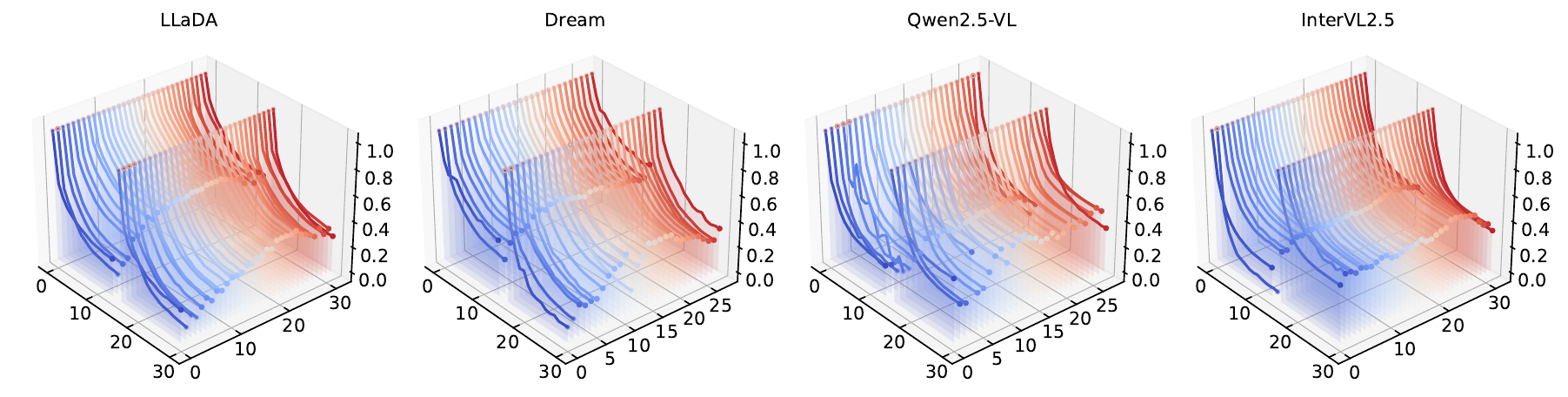}
    \caption{The full losses trained on the calibration dataset for four models. The quantization errors are measured per layer and grouped in terms of different activation types. To better show the differences, the numerical values are normalized. }
    \label{fig:losses_norm}
\end{figure*}

\begin{figure*}
    \centering
    \includegraphics[width=\linewidth]{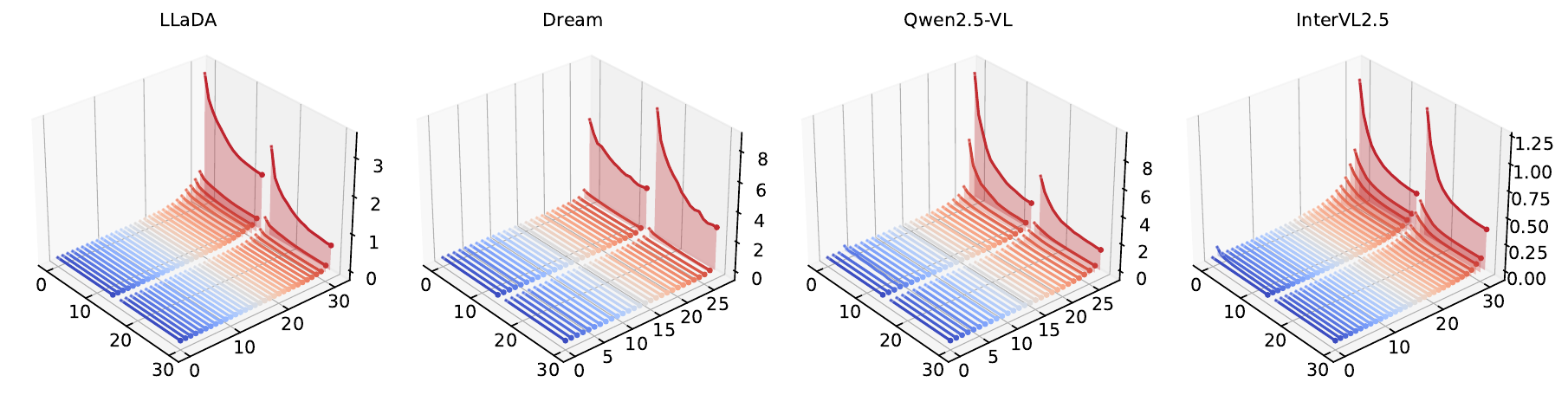}
    \caption{The full losses trained on the calibration dataset for four models. The quantization errors are measured per layer and grouped in terms of different activation types.}
    \label{fig:losses}
\end{figure*}

\subsection{Activation Visualization}
\label{appendix:activation_vis}

We provide the full details of activation distributions for all layers across four evaluated models, as shown in Figures~\ref{fig:llada_activations}, ~\ref{fig:dream_activations}, ~\ref{fig:qwen25vl_activations}, and ~\ref{fig:internvl25_activations}. We also provide the full quantization errors across four models in Figures~\ref{fig:losses_norm} and~\ref{fig:losses}.

\begin{figure}
    \centering
    \includegraphics[width=\linewidth]{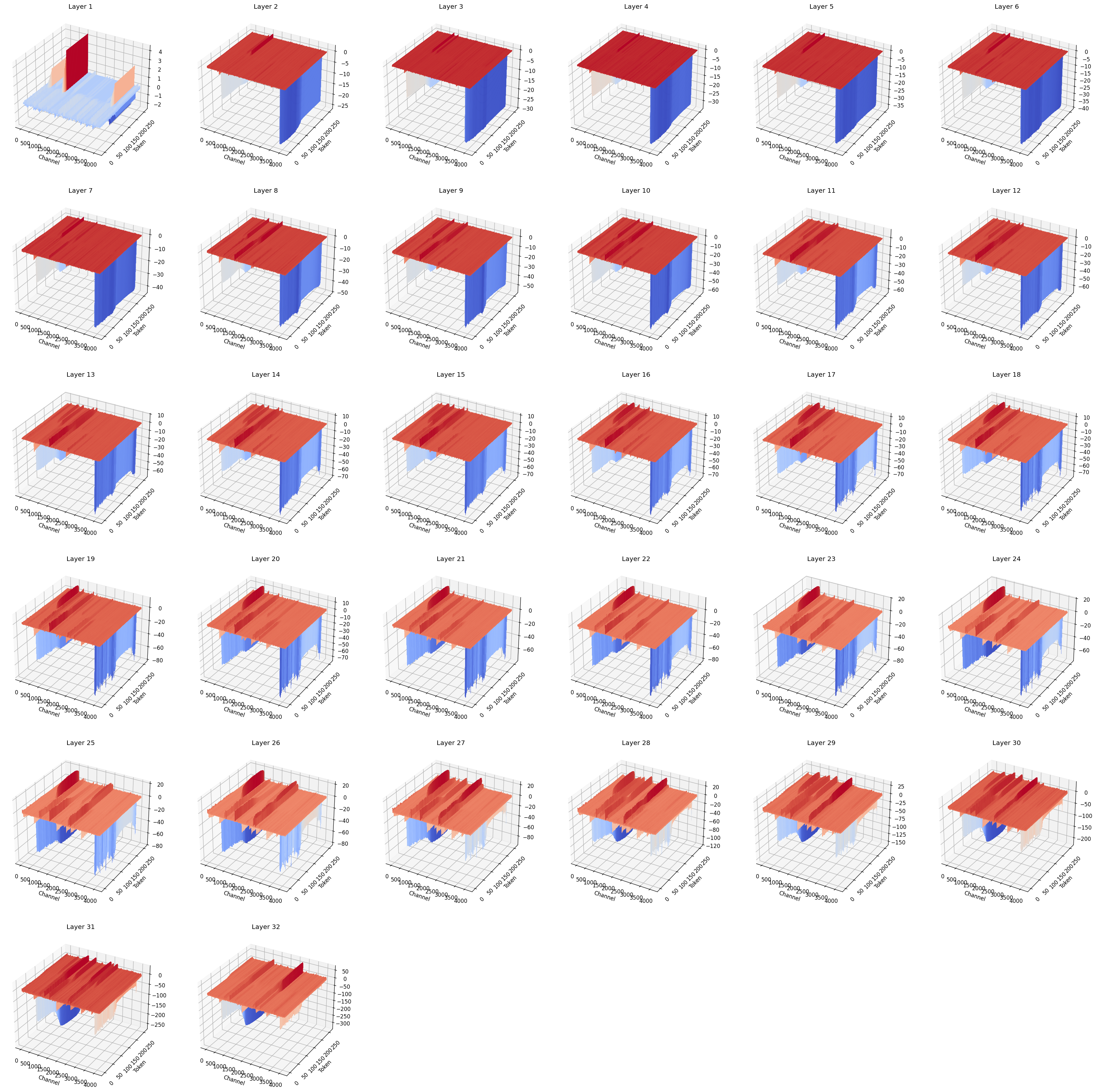}
    \caption{Illustrative examples of activation distributions across different layers on \textit{LLaDA-8B-Instruct}. There is a clear dividing line between unmasked and masked tokens.}
    \label{fig:llada_activations}
\end{figure}

\begin{figure}
    \centering
    \includegraphics[width=\linewidth]{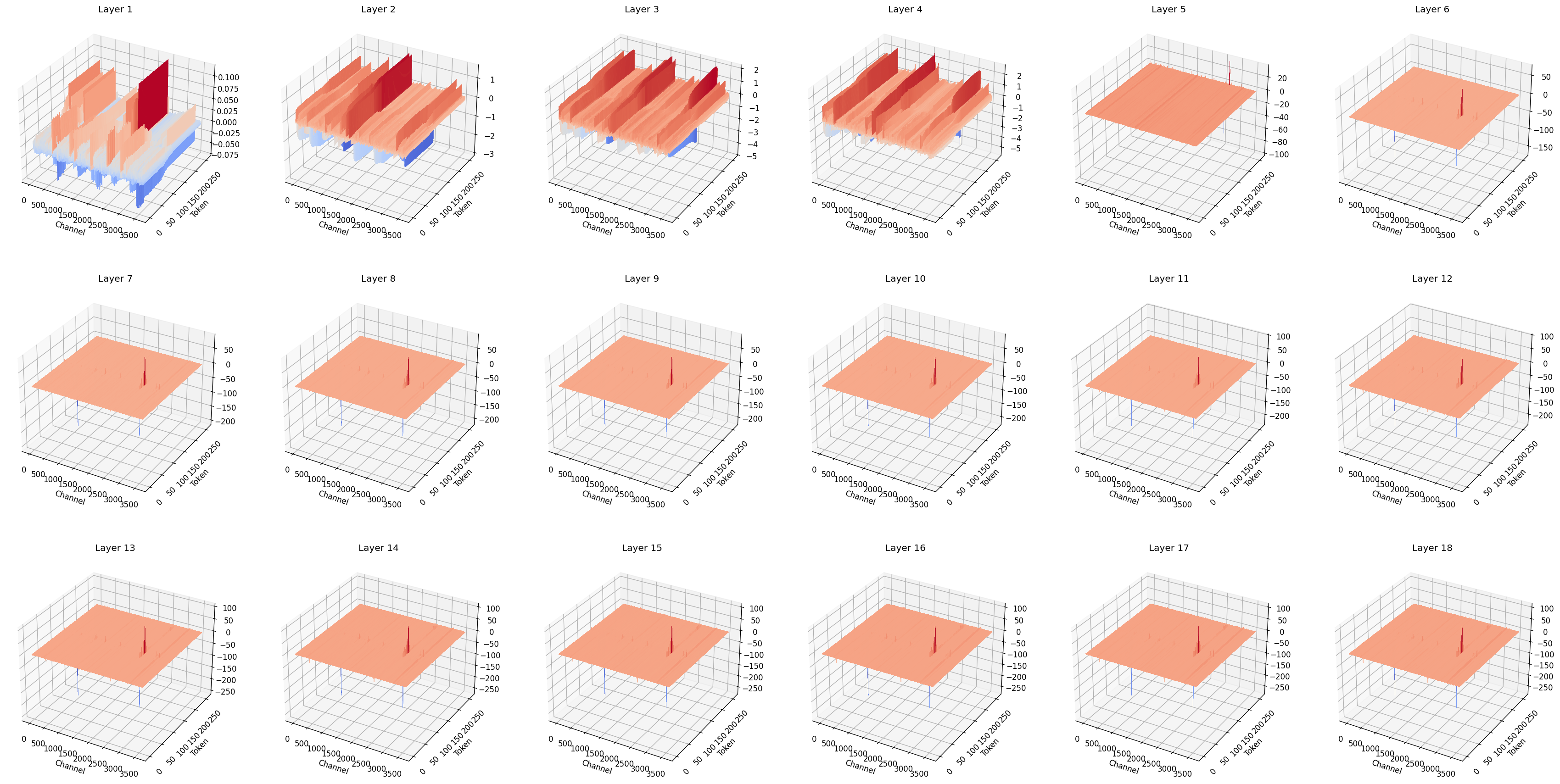}
    \caption{Illustrative examples of activation distributions across different layers on \textit{Dream-v0-Instruct-7B}. There is a clear dividing line and spikes to distinguish unmasked and masked tokens.}
    \label{fig:dream_activations}
\end{figure}

\begin{figure}
    \centering
    \includegraphics[width=\linewidth]{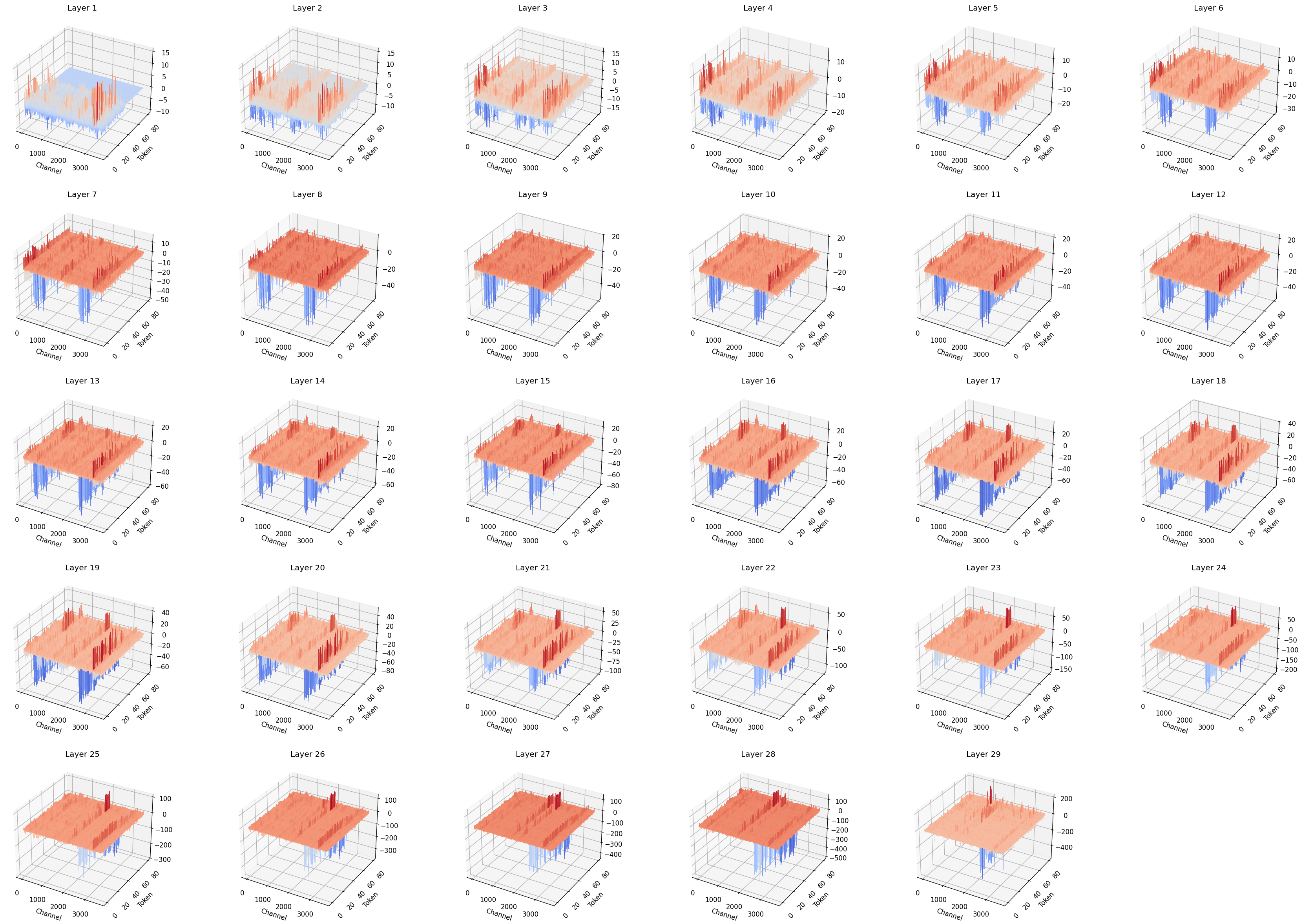}
    \caption{Illustrative examples of activation distributions across different layers on \textit{Qwen2.5-VL-7B-Instruct}. Only the language tower is plotted. There is a clear dividing line to distinguish vision and text tokens.}
    \label{fig:qwen25vl_activations}
\end{figure}

\begin{figure}
    \centering
    \includegraphics[width=\linewidth]{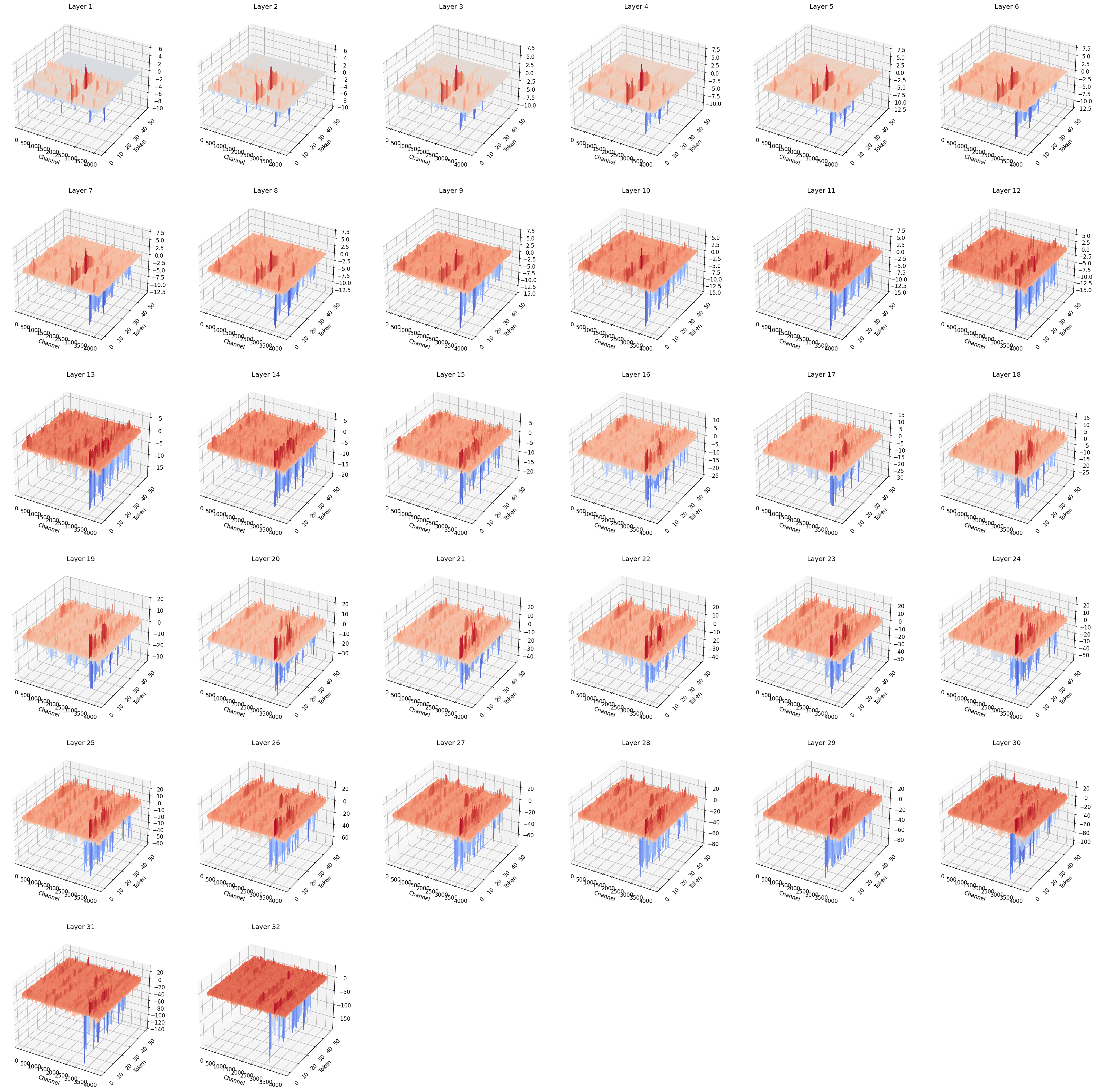}
    \caption{Illustrative examples of activation distributions across different layers on \textit{InternVL2.5-8B}. Only the language tower is plotted. There is a clear dividing line to distinguish vision and text tokens.}
    \label{fig:internvl25_activations}
\end{figure}

\subsection{Case Studies}
\label{appendix:case_studies}

In this section, we provide the cases generated by different models to demonstrate the effectiveness of FreeAct~(see Figures~\ref{fig:case_llada}, \ref{fig:case_dream}, \ref{fig:case_qwenvl}, and \ref{fig:case_internvl}).

\begin{figure}[ht]
\centering
\begin{promptbox}
\textbf{Prompt:} \begin{lstlisting}[
    language=Python,
    basicstyle=\small\ttfamily\color{black},
    keywordstyle=\color{black},
    commentstyle=\color{black},
    stringstyle=\color{black},
    showstringspaces=false,
    breaklines=true,
    numbers=none,
    xleftmargin=0pt
]
Janet s ducks lay 16 eggs per day. She eats three for breakfast every morning and bakes muffins for her friends every day with four. She sells the remainder at the farmers market daily for $2 per fresh duck egg. How much in dollars does she make every day at the farmers market?
\end{lstlisting}

\vspace{0.2em}
\textbf{Response:} 
\begin{baselinebox}
\raggedright
\begin{lstlisting}[
    language=python,
    basicstyle=\footnotesize\ttfamily\color{darkgreen!80},
    keywordstyle=\color{darkgreen!80},
    commentstyle=\color{darkgreen!80},
    stringstyle=\color{darkgreen!80},
    literate={(model)}{{\color{darkgreen!80}\bfseries\normalsize\rmfamily{BF16}}}{4},
    showstringspaces=false,
    breaklines=true,
    numbers=none,
    xleftmargin=0pt
]
(model) Janet calculates how many eggs she has left after her breakfast and her friends' breakfast. She eats 3 eggs for breakfast and bakes 4 eggs for her friends, so she uses a total of 3 + 4 = 7 eggs per day.\n\nHer's ducks lay 16 eggs per day, so she has 16 - 7 = 9 eggs left.\n\nShe sells these 9 eggs at $2 per egg, so she makes 9 * $2 = $<<9*2=18>>18 per day.\n\nTherefore, Janet makes $18 every day at the farmers' market.\n######## 18
\end{lstlisting}

\noindent\rule{\linewidth}{0.2pt}

\vspace{0.2em}
\begin{lstlisting}[
    language=Python,
    basicstyle=\footnotesize\ttfamily\color{bad},
    keywordstyle=\color{bad},
    commentstyle=\color{bad},
    stringstyle=\color{bad},
    literate={(model)}{{\color{bad}\bfseries\normalsize\rmfamily{RTN (W4A4)}}}{10},
    showstringspaces=false,
    breaklines=true,
    numbers=none,
    xleftmargin=0pt
]
(model)  000000<$ 10%00 100000000000000 100000000000000000000000000000000000000000000000000000000000000000....
\end{lstlisting}
\end{baselinebox}

\vspace{-0.7em}

\begin{freequantbox}
\raggedright
\begin{lstlisting}[
    language=Python,
    basicstyle=\footnotesize\ttfamily\color{darkgreen},
    keywordstyle=\color{darkgreen},
    commentstyle=\color{darkgreen},
    stringstyle=\color{darkgreen},
    literate={(model)}{{\color{darkgreen}\bfseries\normalsize\rmfamily{FreeAct (W4A4)}}}{15},
    showstringspaces=false,
    breaklines=true,
    numbers=none,
    xleftmargin=0pt
]
(model) Janet calculates how many eggs she has left after her breakfast and her friends' breakfast. She eats 3 eggs for breakfast and bakes 4 eggs for her friends, so she uses a total of 3 + 4 = 7 eggs per day.\n\nHer's ducks lay 16 eggs per day, so she has 16 - 7 = 9 eggs left.\n\nShe sells these 9 eggs at $2 per egg, so she makes 9 * $2 = $<<9*2=18>>18 per day.\n\nTherefore, Janet makes $18 every day at the farmers' market.\n######## 18
\end{lstlisting}
\end{freequantbox}
\end{promptbox}
\caption{Case study of how FreeAct integrates into \textit{LLaDA} and performs under the W4A4 quantization setting on the GSM8K dataset, compared with the RTN method, which fails to respond. }
\label{fig:case_llada}
\end{figure}

\begin{figure}[ht]
\centering
\begin{promptbox}
\textbf{Prompt:} \begin{lstlisting}[
    language=python,
    basicstyle=\small\ttfamily\color{black},
    keywordstyle=\color{black},
    commentstyle=\color{black},
    stringstyle=\color{black},
    showstringspaces=false,
    breaklines=true,
    numbers=none,
    xleftmargin=0pt
]
Janet s ducks lay 16 eggs per day. She eats three for breakfast every morning and bakes muffins for her friends every day with four. She sells the remainder at the farmers market daily for $2 per fresh duck egg. How much in dollars does she make every day at the farmers market?
\end{lstlisting}

\vspace{0.2em}
\textbf{Response:} 
\begin{baselinebox}
\raggedright
\begin{lstlisting}[
    language=python,
    basicstyle=\footnotesize\ttfamily\color{darkgreen!80},
    keywordstyle=\color{darkgreen!80},
    commentstyle=\color{darkgreen!80},
    stringstyle=\color{darkgreen!80},
    literate={(model)}{{\color{darkgreen!80}\bfseries\normalsize\rmfamily{BF16}}}{4},
    showstringspaces=false,
    breaklines=true,
    numbers=none,
    xleftmargin=0pt
]
(model)  Janet eats 3 eggs for breakfast and bakes 4 eggs for muffins, so she uses 3 + 4 = <<3+4=7>>7 eggs per day.\nShe then sells 16 - 7 = <<16-7=9>>9 eggs per day.\nShe sells 9 eggs for $2 each, so she makes 9 * $2 = <<9*2=18>>$18 per day at the farmers' market.\n#### 18
\end{lstlisting}

\noindent\rule{\linewidth}{0.2pt}

\vspace{0.2em}
\begin{lstlisting}[
    language=Python,
    basicstyle=\footnotesize\ttfamily\color{bad},
    keywordstyle=\color{bad},
    commentstyle=\color{bad},
    stringstyle=\color{bad},
    literate={(model)}{{\color{bad}\bfseries\normalsize\rmfamily{RTN (W4A4)}}}{10},
    showstringspaces=false,
    breaklines=true,
    numbers=none,
    xleftmargin=0pt
]
(model)    relico00 flashback relicl musebrates    Lucifer:   relic relic relic  inn1P1 relic( relic relic0  brates lodge0 relicp relic s hop relic relich pregnancies relic1 relicog relic1 relic\n relic relic nip skate hen conco
\end{lstlisting}
\end{baselinebox}

\vspace{-0.7em}

\begin{freequantbox}
\raggedright
\begin{lstlisting}[
    language=Python,
    basicstyle=\footnotesize\ttfamily\color{darkgreen},
    keywordstyle=\color{darkgreen},
    commentstyle=\color{darkgreen},
    stringstyle=\color{darkgreen},
    literate={(model)}{{\color{darkgreen}\bfseries\normalsize\rmfamily{FreeAct (W4A4)}}}{15},
    showstringspaces=false,
    breaklines=true,
    numbers=none,
    xleftmargin=0pt
]
(model)  Janet eats 3 eggs for breakfast and bakes 4 eggs for muffins, so she uses 3 + 4 = <<3+4=7>>7 eggs per day.\nShe then sells 16 - 7 = <<16-7=9>>9 eggs per day.\nShe sells 9 eggs for $2 each, so she makes 9 * $2 = <<9*2=18>>$18 per day at the farmers' market.\n#### 18
\end{lstlisting}
\end{freequantbox}
\end{promptbox}
\caption{Case study of how FreeAct integrates into \textit{Dream} and performs under the W4A4 quantization setting on the GSM8K dataset, compared with the RTN method, which fails to respond. }
\label{fig:case_dream}
\end{figure}

\begin{figure}[ht]
\centering
\begin{promptbox}
\textbf{Prompt:} \begin{lstlisting}[
    language=Python,
    basicstyle=\small\ttfamily\color{black},
    keywordstyle=\color{black},
    commentstyle=\color{black},
    stringstyle=\color{black},
    showstringspaces=false,
    breaklines=true,
    numbers=none,
    xleftmargin=0pt
]
<image 1> Baxter Company has a relevant range of production between 15,000 and 30,000 units. The following cost data represents average variable costs per unit for 25,000 units of production. If 30,000 units are produced, what are the per unit manufacturing overhead costs incurred?\nA. $6\nB. $7\nC. $8\nD. $9\n\nAnswer with the option's letter from the given choices directly.
\end{lstlisting}

\vspace{0.2em}
\textbf{Response:} 
\begin{baselinebox}
\raggedright
\begin{lstlisting}[
    language=python,
    basicstyle=\footnotesize\ttfamily\color{darkgreen!80},
    keywordstyle=\color{darkgreen!80},
    commentstyle=\color{darkgreen!80},
    stringstyle=\color{darkgreen!80},
    literate={(model)}{{\color{darkgreen!80}\bfseries\normalsize\rmfamily{BF16}}}{4},
    showstringspaces=false,
    breaklines=true,
    numbers=none,
    xleftmargin=0pt
]
(model)  B
\end{lstlisting}

\noindent\rule{\linewidth}{0.2pt}

\vspace{0.2em}
\begin{lstlisting}[
    language=Python,
    basicstyle=\footnotesize\ttfamily\color{bad},
    keywordstyle=\color{bad},
    commentstyle=\color{bad},
    stringstyle=\color{bad},
    literate={(model)}{{\color{bad}\bfseries\normalsize\rmfamily{RTN (W4A4)}}}{10},
    showstringspaces=false,
    breaklines=true,
    numbers=none,
    xleftmargin=0pt
]
(model)  Gluten  \n\n\n  2 -Pack-gend    Stainless ( 1   Spare Bottle-Pack    Clamp  \n 10-Packerrated  Mounted1 Movement A S 1/&ombine S1   1 4 Aluminum  0 4   0  0 Wholesale  Decorating  Color  2-Packlasses Clamp 1 1S lideh1-Pack Qty5-Pack  Polymerdden  Cabinets1gend-Packreff-Pack -Pack   ringe -Pack Bulk  Lines2
\end{lstlisting}
\end{baselinebox}

\vspace{-0.7em}

\begin{freequantbox}
\raggedright
\begin{lstlisting}[
    language=Python,
    basicstyle=\footnotesize\ttfamily\color{darkgreen},
    keywordstyle=\color{darkgreen},
    commentstyle=\color{darkgreen},
    stringstyle=\color{darkgreen},
    literate={(model)}{{\color{darkgreen}\bfseries\normalsize\rmfamily{FreeAct (W4A4)}}}{15},
    showstringspaces=false,
    breaklines=true,
    numbers=none,
    xleftmargin=0pt
]
(model) B
\end{lstlisting}
\end{freequantbox}
\end{promptbox}
\caption{Case study of how FreeAct integrates into \textit{QWen2.5 VL} and performs under the W4A4 quantization setting on the MMMU dataset, compared with the RTN method, which fails to respond. }
\label{fig:case_qwenvl}
\end{figure}

\begin{figure}[ht]
\centering
\begin{promptbox}
\textbf{Prompt:} \begin{lstlisting}[
    language=Python,
    basicstyle=\small\ttfamily\color{black},
    keywordstyle=\color{black},
    commentstyle=\color{black},
    stringstyle=\color{black},
    showstringspaces=false,
    breaklines=true,
    numbers=none,
    xleftmargin=0pt
]
<image 1> Baxter Company has a relevant range of production between 15,000 and 30,000 units. The following cost data represents average variable costs per unit for 25,000 units of production. If 30,000 units are produced, what are the per unit manufacturing overhead costs incurred?\nA. $6\nB. $7\nC. $8\nD. $9\n\nAnswer with the option's letter from the given choices directly.
\end{lstlisting}

\vspace{0.2em}
\textbf{Response:} 
\begin{baselinebox}
\raggedright
\begin{lstlisting}[
    language=python,
    basicstyle=\footnotesize\ttfamily\color{darkgreen!80},
    keywordstyle=\color{darkgreen!80},
    commentstyle=\color{darkgreen!80},
    stringstyle=\color{darkgreen!80},
    literate={(model)}{{\color{darkgreen!80}\bfseries\normalsize\rmfamily{BF16}}}{4},
    showstringspaces=false,
    breaklines=true,
    numbers=none,
    xleftmargin=0pt
]
(model) B
\end{lstlisting}

\noindent\rule{\linewidth}{0.2pt}

\vspace{0.2em}
\begin{lstlisting}[
    language=Python,
    basicstyle=\footnotesize\ttfamily\color{bad},
    keywordstyle=\color{bad},
    commentstyle=\color{bad},
    stringstyle=\color{bad},
    literate={(model)}{{\color{bad}\bfseries\normalsize\rmfamily{RTN (W4A4)}}}{10},
    showstringspaces=false,
    breaklines=true,
    numbers=none,
    xleftmargin=0pt
]
(model) In n 1, 1 5,000-00 0 s/ 1.5/0. 5/2. The 2010,000 sp?/ 1.5.0 are then 5/5/10*5 10/0/10/ 5/6?/ $The The average costs? / 10/6/5.0?. The The ten /5/5 The average cost/ 5,000/2/5 The average cost per10? /5%? The average? 5/5 5/0/5?/0
\end{lstlisting}
\end{baselinebox}

\vspace{-0.7em}

\begin{freequantbox}
\raggedright
\begin{lstlisting}[
    language=Python,
    basicstyle=\footnotesize\ttfamily\color{darkgreen},
    keywordstyle=\color{darkgreen},
    commentstyle=\color{darkgreen},
    stringstyle=\color{darkgreen},
    literate={(model)}{{\color{darkgreen}\bfseries\normalsize\rmfamily{FreeAct (W4A4)}}}{14},
    showstringspaces=false,
    breaklines=true,
    numbers=none,
    xleftmargin=0pt
]
(model) B. $7
\end{lstlisting}
\end{freequantbox}
\end{promptbox}
\caption{Case study of how FreeAct integrates into \textit{InternVL2.5} and performs under the W4A4 quantization setting on the MMMU dataset, compared with the RTN method, which fails to respond. }
\label{fig:case_internvl}
\end{figure}

\section{Reproducibility}
We provide implementation details, involving illustrative algorithm descriptions in Appendix~\ref{appendix:implementation}  and pseudo-code in Section~\ref{sec:freeact}. The source code will be publicly released for reproducibility. 

\section{Limitations}
We demonstrate the effectiveness of FreeAct through rigorous empirical validation and theoretical support, establishing its conceptual advancement in the field of quantization. Despite these strengths, several limitations remain for future exploration. 1) The current version of FreeAct cannot accept more modalities or diverse input types ($\# > 2$). It can be extended to more complex scenarios, but requires an in-depth investigation into activation distributions.
2) Due to the resource constraints, we do not scale the model size.
3) We do not deploy FreeAct offline while evaluating its effectiveness with fake quantization.

\section{Use of LLMs in Writing}\label{sec:llm_use}
We used an LLM solely to polish the writing and correct grammatical issues during the preparation of this submission. The LLM was not involved in any idea generation, experiment design, or analysis. All scientific contributions are entirely made by the authors.

%% file: sections/related_work.tex
\section{Related Work}
\label{sec:related_work}
\subsection{Multimodal and Diffusion LLMs}
Large Language Models (LLMs)~\cite{achiam2023gpt, touvron2023llama,thirunavukarasu2023large,zhao2023survey,chang2024survey,kaddour2023challenges} have demonstrated remarkable success in generating textual tokens, paving the way for advanced extensions such as Multimodal LLMs (MLLMs)~\cite{team2023gemini,luo2025next,xiao2026mimo,zhang2025mimo,sapkota2025multi,lyu2023macaw,wang2024longllava}, which integrate modalities like vision, and diffusion LLMs (dLLMs)~\cite{nie2025large,ye2025dream,wu2025fast,jin2025role,tian2025next}, which introduce a denoising mechanism on masked tokens to enable simultaneous multi-token prediction. Typically, MLLMs incorporate a visual encoder to process visual inputs, followed by a projector to align different modalities with a powerful decoder-based LLM. Since the release of ChatGPT, researchers augment open-source LLMs (\textit{e.g.}, the LLaMA series~\cite{touvron2023llama}) with visual perception capabilities, establishing a widely adopted paradigm exemplified by MiniGPT-4~\cite{zhu2023minigpt} and NExT-GPT~\cite{wu2024next}. Subsequent research has scaled data and parameters, leading to the release of Qwen-VL~\cite{bai2025qwen2}, InternVL~\cite{chen2024internvl}, and MiniCPM-V~\cite{yao2024minicpm}. Most recent works propose a unified multimodal paradigm that eliminates the separate vision projector, opting instead for a tightly coupled visual-text architecture capable of both understanding and generation. Meanwhile, dLLMs explore the prediction of multiple masked tokens, inspired by prior discrete diffusion modeling with absorbing states~\cite{nie2025large}. LLaDA~\cite{nie2025large} and Dream~\cite{ye2025dream} successfully implement this diffusion process in billion-parameter models, achieving performance comparable to purely autoregressive LLMs. In this paper, we target these two advanced paradigms and formulate a unified quantization framework for both.

\subsection{Quantization for LLMs}
Quantization significantly improves efficiency by reducing the precision, particularly for LLMs. Quantization-Aware Training (QAT) adopts low-bit precision directly during training~\cite{nagel2022overcoming,chen2025efficientqat,liu2024llm,bondarenko2024low, du2024bitdistiller}, whereas Post-Training Quantization (PTQ) is applied to pre-trained models, aiming to retain their original capabilities in a resource-efficient manner~\cite{yao2023comprehensive}. Within the PTQ paradigm, recent research is generally categorized into weight-only~\cite{frantar2022gptq, lin2024awq} and weight-activation quantization~\cite{tseng2024quip, ashkboos2024quarot, lin2024duquant}, where the latter further quantizes activations to achieve higher efficiency but presents greater challenges. Several works~\cite{shao2023omniquant, bhalgat2020lsq+} focus on managing clip thresholds and fine-grained scales to constrain value ranges. In contrast, SmoothQuant~\cite{xiao2023smoothquant} migrates the difficulty of quantization from activations to weights using a diagonal matrix, while QuaRot~\cite{ashkboos2024quarot} introduces the Hadamard transformation to mitigate outliers, establishing a one-to-one transformation that ensures mathematical equivalence. Recent studies extend this transformation-based paradigm by making these matrices learnable~\cite{lin2024duquant, hu2025ostquant} and more flexible, evolving from simple rotations~\cite{liu2024spinquant} to general affine transformations~\cite{ma2024affinequant}. Notably, FlatQuant~\cite{sun2024flatquant} incorporates the Kronecker product with affine transformations, achieving state-of-the-art quantization quality on LLMs. Researchers have also specialized these techniques for dLLMs~\cite{xu2025dllmquant, lin2025quantization, zhang2025quant} and MLLMs~\cite{yu2025mquant}, which involve different token types. However, these efforts mainly focus on optimizing scales and calibrating the dataset, setting transformation advancement aside. In this paper, we pioneer the exploration of flexible transformations beyond one-to-one for LLMs, specifically dLLMs and MLLMs, that process diverse input types with distinct activation distributions.

%% file: reference.bib
@article{loshchilov2017decoupled,
  title={Decoupled weight decay regularization},
  author={Loshchilov, Ilya and Hutter, Frank},
  journal={arXiv preprint arXiv:1711.05101},
  year={2017}
}

@article{achiam2023gpt,
  title={Gpt-4 technical report},
  author={Achiam, Josh and Adler, Steven and Agarwal, Sandhini and Ahmad, Lama and Akkaya, Ilge and Aleman, Florencia Leoni and Almeida, Diogo and Altenschmidt, Janko and Altman, Sam and Anadkat, Shyamal and others},
  journal={arXiv preprint arXiv:2303.08774},
  year={2023}
}

@article{touvron2023llama,
  title={Llama: Open and efficient foundation language models},
  author={Touvron, Hugo and Lavril, Thibaut and Izacard, Gautier and Martinet, Xavier and Lachaux, Marie-Anne and Lacroix, Timoth{\'e}e and Rozi{\`e}re, Baptiste and Goyal, Naman and Hambro, Eric and Azhar, Faisal and others},
  journal={arXiv preprint arXiv:2302.13971},
  year={2023}
}

@article{yang2025qwen3,
  title={Qwen3 technical report},
  author={Yang, An and Li, Anfeng and Yang, Baosong and Zhang, Beichen and Hui, Binyuan and Zheng, Bo and Yu, Bowen and Gao, Chang and Huang, Chengen and Lv, Chenxu and others},
  journal={arXiv preprint arXiv:2505.09388},
  year={2025}
}

@article{grattafiori2024llama,
  title={The llama 3 herd of models},
  author={Grattafiori, Aaron and Dubey, Abhimanyu and Jauhri, Abhinav and Pandey, Abhinav and Kadian, Abhishek and Al-Dahle, Ahmad and Letman, Aiesha and Mathur, Akhil and Schelten, Alan and Vaughan, Alex and others},
  journal={arXiv preprint arXiv:2407.21783},
  year={2024}
}

@inproceedings{egashira2024exploiting,
  title={Exploiting llm quantization},
  author={Egashira, Kazuki and Vero, Mark and Staab, Robin and He, Jingxuan and Vechev, Martin},
  booktitle={NeurIPS},
  pages={41709--41732},
  year={2024}
}

@article{zhou2024survey,
  title={A survey on efficient inference for large language models},
  author={Zhou, Zixuan and Ning, Xuefei and Hong, Ke and Fu, Tianyu and Xu, Jiaming and Li, Shiyao and Lou, Yuming and Wang, Luning and Yuan, Zhihang and Li, Xiuhong and others},
  journal={arXiv preprint arXiv:2404.14294},
  year={2024}
}

@inproceedings{lin2024awq,
  title={Awq: Activation-aware weight quantization for on-device llm compression and acceleration},
  author={Lin, Ji and Tang, Jiaming and Tang, Haotian and Yang, Shang and Chen, Wei-Ming and Wang, Wei-Chen and Xiao, Guangxuan and Dang, Xingyu and Gan, Chuang and Han, Song},
  booktitle={MLSys},
  pages={87--100},
  year={2024}
}

@article{tseng2024quip,
  title={Quip\#: Even better llm quantization with hadamard incoherence and lattice codebooks},
  author={Tseng, Albert and Chee, Jerry and Sun, Qingyao and Kuleshov, Volodymyr and De Sa, Christopher},
  journal={ICML},
  year={2024}
}

@article{kaplan2020scaling,
  title={Scaling laws for neural language models},
  author={Kaplan, Jared and McCandlish, Sam and Henighan, Tom and Brown, Tom B and Chess, Benjamin and Child, Rewon and Gray, Scott and Radford, Alec and Wu, Jeffrey and Amodei, Dario},
  journal={arXiv preprint arXiv:2001.08361},
  year={2020}
}

@inproceedings{chee2023quip,
  title={Quip: 2-bit quantization of large language models with guarantees},
  author={Chee, Jerry and Cai, Yaohui and Kuleshov, Volodymyr and De Sa, Christopher M},
  booktitle={NeurIPS},
  pages={4396--4429},
  year={2023}
}

@inproceedings{ashkboos2024quarot,
  title={Quarot: Outlier-free 4-bit inference in rotated llms},
  author={Ashkboos, Saleh and Mohtashami, Amirkeivan and Croci, Maximilian L and Li, Bo and Cameron, Pashmina and Jaggi, Martin and Alistarh, Dan and Hoefler, Torsten and Hensman, James},
  booktitle={NeurIPS},
  pages={100213--100240},
  year={2024}
}

@article{sun2024flatquant,
  title={Flatquant: Flatness matters for llm quantization},
  author={Sun, Yuxuan and Liu, Ruikang and Bai, Haoli and Bao, Han and Zhao, Kang and Li, Yuening and Hu, Jiaxin and Yu, Xianzhi and Hou, Lu and Yuan, Chun and others},
  journal={ICML},
  year={2025}
}

@inproceedings{liu2024spinquant,
  title={Spinquant: Llm quantization with learned rotations},
  author={Liu, Zechun and Zhao, Changsheng and Fedorov, Igor and Soran, Bilge and Choudhary, Dhruv and Krishnamoorthi, Raghuraman and Chandra, Vikas and Tian, Yuandong and Blankevoort, Tijmen},
  booktitle={ICLR},
  year={2024}
}

@article{nie2025large,
  title={Large language diffusion models},
  author={Nie, Shen and Zhu, Fengqi and You, Zebin and Zhang, Xiaolu and Ou, Jingyang and Hu, Jun and Zhou, Jun and Lin, Yankai and Wen, Ji-Rong and Li, Chongxuan},
  journal={arXiv preprint arXiv:2502.09992},
  year={2025}
}

@article{ye2025dream,
  title={Dream 7b: Diffusion large language models},
  author={Ye, Jiacheng and Xie, Zhihui and Zheng, Lin and Gao, Jiahui and Wu, Zirui and Jiang, Xin and Li, Zhenguo and Kong, Lingpeng},
  journal={arXiv preprint arXiv:2508.15487},
  year={2025}
}

@article{bai2025qwen2,
  title={Qwen2. 5-vl technical report},
  author={Bai, Shuai and Chen, Keqin and Liu, Xuejing and Wang, Jialin and Ge, Wenbin and Song, Sibo and Dang, Kai and Wang, Peng and Wang, Shijie and Tang, Jun and others},
  journal={arXiv preprint arXiv:2502.13923},
  year={2025}
}

@inproceedings{lin2025qserve,
  title={Qserve: W4a8kv4 quantization and system co-design for efficient llm serving},
  author={Lin, Yujun and Tang, Haotian and Yang, Shang and Zhang, Zhekai and Xiao, Guangxuan and Gan, Chuang and Han, Song},
  booktitle={MLSys},
  year={2025}
}

@article{guo2025seed1,
  title={Seed1. 5-vl technical report},
  author={Guo, Dong and Wu, Faming and Zhu, Feida and Leng, Fuxing and Shi, Guang and Chen, Haobin and Fan, Haoqi and Wang, Jian and Jiang, Jianyu and Wang, Jiawei and others},
  journal={arXiv preprint arXiv:2505.07062},
  year={2025}
}

@article{xiao2026mimo,
  title={MiMo-V2-Flash Technical Report},
  author={Xiao, Bangjun and Xia, Bingquan and Yang, Bo and Gao, Bofei and Shen, Bowen and Zhang, Chen and He, Chenhong and Lou, Chiheng and Luo, Fuli and Wang, Gang and others},
  journal={arXiv preprint arXiv:2601.02780},
  year={2026}
}

@article{tian2025next,
  title={From next-token to next-block: A principled adaptation path for diffusion llms},
  author={Tian, Yuchuan and Liang, Yuchen and Sun, Jiacheng and Zhang, Shuo and Yang, Guangwen and Shu, Yingte and Fang, Sibo and Guo, Tianyu and Han, Kai and Xu, Chao and others},
  journal={arXiv preprint arXiv:2512.06776},
  year={2025}
}

@article{jin2025role,
  title={On the Role of Discreteness in Diffusion LLMs},
  author={Jin, Ziqi and Wang, Bin and Lin, Xiang and Bing, Lidong and Sun, Aixin},
  journal={arXiv preprint arXiv:2512.22630},
  year={2025}
}

@article{wu2025fast,
  title={Fast-dllm v2: Efficient block-diffusion llm},
  author={Wu, Chengyue and Zhang, Hao and Xue, Shuchen and Diao, Shizhe and Fu, Yonggan and Liu, Zhijian and Molchanov, Pavlo and Luo, Ping and Han, Song and Xie, Enze},
  journal={arXiv preprint arXiv:2509.26328},
  year={2025}
}

@article{zhang2025mimo,
  title={MiMo-Audio: Audio Language Models are Few-Shot Learners},
  author={Zhang, Dong and Wang, Gang and Xue, Jinlong and Fang, Kai and Zhao, Liang and Ma, Rui and Ren, Shuhuai and Liu, Shuo and Guo, Tao and Zhuang, Weiji and others},
  journal={arXiv preprint arXiv:2512.23808},
  year={2025}
}

@article{team2023gemini,
  title={Gemini: a family of highly capable multimodal models},
  author={Team, Gemini and Anil, Rohan and Borgeaud, Sebastian and Alayrac, Jean-Baptiste and Yu, Jiahui and Soricut, Radu and Schalkwyk, Johan and Dai, Andrew M and Hauth, Anja and Millican, Katie and others},
  journal={arXiv preprint arXiv:2312.11805},
  year={2023}
}

@article{luo2025next,
  title={Next-omni: Towards any-to-any omnimodal foundation models with discrete flow matching},
  author={Luo, Run and Xia, Xiaobo and Wang, Lu and Chen, Longze and Shan, Renke and Luo, Jing and Yang, Min and Chua, Tat-Seng},
  journal={arXiv preprint arXiv:2510.13721},
  year={2025}
}

@article{cobbe2021training,
  title={Training verifiers to solve math word problems},
  author={Cobbe, Karl and Kosaraju, Vineet and Bavarian, Mohammad and Chen, Mark and Jun, Heewoo and Kaiser, Lukasz and Plappert, Matthias and Tworek, Jerry and Hilton, Jacob and Nakano, Reiichiro and others},
  journal={arXiv preprint arXiv:2110.14168},
  year={2021}
}

@article{chang2024survey,
  title={A survey on evaluation of large language models},
  author={Chang, Yupeng and Wang, Xu and Wang, Jindong and Wu, Yuan and Yang, Linyi and Zhu, Kaijie and Chen, Hao and Yi, Xiaoyuan and Wang, Cunxiang and Wang, Yidong and others},
  journal={ACM Transactions on Intelligent Systems and Technology},
  volume={15},
  number={3},
  pages={1--45},
  year={2024}
}

@article{kaddour2023challenges,
  title={Challenges and applications of large language models},
  author={Kaddour, Jean and Harris, Joshua and Mozes, Maximilian and Bradley, Herbie and Raileanu, Roberta and McHardy, Robert},
  journal={arXiv preprint arXiv:2307.10169},
  year={2023}
}

@article{zhao2023survey,
  title={A survey of large language models},
  author={Zhao, Wayne Xin and Zhou, Kun and Li, Junyi and Tang, Tianyi and Wang, Xiaolei and Hou, Yupeng and Min, Yingqian and Zhang, Beichen and Zhang, Junjie and Dong, Zican and others},
  journal={arXiv preprint arXiv:2303.18223},
  year={2023}
}

@article{thirunavukarasu2023large,
  title={Large language models in medicine},
  author={Thirunavukarasu, Arun James and Ting, Darren Shu Jeng and Elangovan, Kabilan and Gutierrez, Laura and Tan, Ting Fang and Ting, Daniel Shu Wei},
  journal={Nature Medicine},
  volume={29},
  number={8},
  pages={1930--1940},
  year={2023}
}

@article{chen2021evaluating,
  title={Evaluating large language models trained on code},
  author={Chen, Mark},
  journal={arXiv preprint arXiv:2107.03374},
  year={2021}
}

@inproceedings{liu2024mmbench,
  title={Mmbench: Is your multi-modal model an all-around player?},
  author={Liu, Yuan and Duan, Haodong and Zhang, Yuanhan and Li, Bo and Zhang, Songyang and Zhao, Wangbo and Yuan, Yike and Wang, Jiaqi and He, Conghui and Liu, Ziwei and others},
  booktitle={ECCV},
  pages={216--233},
  year={2024}
}

@article{hendrycks2021measuring,
  title={Measuring mathematical problem solving with the math dataset},
  author={Hendrycks, Dan and Burns, Collin and Kadavath, Saurav and Arora, Akul and Basart, Steven and Tang, Eric and Song, Dawn and Steinhardt, Jacob},
  journal={arXiv preprint arXiv:2103.03874},
  year={2021}
}

@inproceedings{yue2024mmmu,
  title={Mmmu: A massive multi-discipline multimodal understanding and reasoning benchmark for expert agi},
  author={Yue, Xiang and Ni, Yuansheng and Zhang, Kai and Zheng, Tianyu and Liu, Ruoqi and Zhang, Ge and Stevens, Samuel and Jiang, Dongfu and Ren, Weiming and Sun, Yuxuan and others},
  booktitle={CVPR},
  pages={9556--9567},
  year={2024}
}

@inproceedings{lin2014microsoft,
  title={Microsoft coco: Common objects in context},
  author={Lin, Tsung-Yi and Maire, Michael and Belongie, Serge and Hays, James and Perona, Pietro and Ramanan, Deva and Doll{\'a}r, Piotr and Zitnick, C Lawrence},
  booktitle={ECCV},
  pages={740--755},
  year={2014}
}

@article{merity2016pointer,
  title={Pointer sentinel mixture models},
  author={Merity, Stephen and Xiong, Caiming and Bradbury, James and Socher, Richard},
  journal={arXiv preprint arXiv:1609.07843},
  year={2016}
}

@article{yao2023comprehensive,
  title={A comprehensive study on post-training quantization for large language models},
  author={Yao, Zhewei and Li, Cheng and Wu, Xiaoxia and Youn, Stephen and He, Yuxiong},
  journal={arXiv preprint arXiv:2303.08302},
  year={2023}
}

@article{zhu2023minigpt,
  title={Minigpt-4: Enhancing vision-language understanding with advanced large language models},
  author={Zhu, Deyao and Chen, Jun and Shen, Xiaoqian and Li, Xiang and Elhoseiny, Mohamed},
  journal={arXiv preprint arXiv:2304.10592},
  year={2023}
}

@article{yao2024minicpm,
  title={Minicpm-v: A gpt-4v level mllm on your phone},
  author={Yao, Yuan and Yu, Tianyu and Zhang, Ao and Wang, Chongyi and Cui, Junbo and Zhu, Hongji and Cai, Tianchi and Li, Haoyu and Zhao, Weilin and He, Zhihui and others},
  journal={arXiv preprint arXiv:2408.01800},
  year={2024}
}

@inproceedings{wu2024next,
  title={Next-gpt: Any-to-any multimodal llm},
  author={Wu, Shengqiong and Fei, Hao and Qu, Leigang and Ji, Wei and Chua, Tat-Seng},
  booktitle={ICML},
  year={2024}
}

@article{bondarenko2024low,
  title={Low-rank quantization-aware training for llms},
  author={Bondarenko, Yelysei and Del Chiaro, Riccardo and Nagel, Markus},
  journal={arXiv preprint arXiv:2406.06385},
  year={2024}
}

@inproceedings{liu2024llm,
  title={Llm-qat: Data-free quantization aware training for large language models},
  author={Liu, Zechun and Oguz, Barlas and Zhao, Changsheng and Chang, Ernie and Stock, Pierre and Mehdad, Yashar and Shi, Yangyang and Krishnamoorthi, Raghuraman and Chandra, Vikas},
  booktitle={ACL Findings},
  pages={467--484},
  year={2024}
}

@inproceedings{chen2025efficientqat,
  title={Efficientqat: Efficient quantization-aware training for large language models},
  author={Chen, Mengzhao and Shao, Wenqi and Xu, Peng and Wang, Jiahao and Gao, Peng and Zhang, Kaipeng and Luo, Ping},
  booktitle={ACL},
  pages={10081--10100},
  year={2025}
}

@inproceedings{nagel2022overcoming,
  title={Overcoming oscillations in quantization-aware training},
  author={Nagel, Markus and Fournarakis, Marios and Bondarenko, Yelysei and Blankevoort, Tijmen},
  booktitle={ICML},
  pages={16318--16330},
  year={2022}
}

@article{liu2025paretoq,
  title={Paretoq: Scaling laws in extremely low-bit llm quantization},
  author={Liu, Zechun and Zhao, Changsheng and Huang, Hanxian and Chen, Sijia and Zhang, Jing and Zhao, Jiawei and Roy, Scott and Jin, Lisa and Xiong, Yunyang and Shi, Yangyang and others},
  journal={arXiv preprint arXiv:2502.02631},
  year={2025}
}

@article{wang2025diffusion,
  title={Diffusion llms can do faster-than-ar inference via discrete diffusion forcing},
  author={Wang, Xu and Xu, Chenkai and Jin, Yijie and Jin, Jiachun and Zhang, Hao and Deng, Zhijie},
  journal={arXiv preprint arXiv:2508.09192},
  year={2025}
}

@article{liu2024deepseek,
  title={Deepseek-v3 technical report},
  author={Liu, Aixin and Feng, Bei and Xue, Bing and Wang, Bingxuan and Wu, Bochao and Lu, Chengda and Zhao, Chenggang and Deng, Chengqi and Zhang, Chenyu and Ruan, Chong and others},
  journal={arXiv preprint arXiv:2412.19437},
  year={2024}
}

@inproceedings{xia2024efficient,
  title={Efficient multi-task llm quantization and serving for multiple lora adapters},
  author={Xia, Yifei and Fu, Fangcheng and Zhang, Wentao and Jiang, Jiawei and Cui, Bin},
  booktitle={NeurIPS},
  pages={63686--63714},
  year={2024}
}

@article{guo2025deepseek,
  title={Deepseek-r1: Incentivizing reasoning capability in llms via reinforcement learning},
  author={Guo, Daya and Yang, Dejian and Zhang, Haowei and Song, Junxiao and Zhang, Ruoyu and Xu, Runxin and Zhu, Qihao and Ma, Shirong and Wang, Peiyi and Bi, Xiao and others},
  journal={arXiv preprint arXiv:2501.12948},
  year={2025}
}

@article{dai2025flashdecoding++next,
  title={FlashDecoding++Next: High Throughput LLM Inference with Latency and Memory Optimization},
  author={Dai, Guohao and Hong, Ke and Mao, Qiuli and Li, Xiuhong and Xu, Jiaming and Huang, Haofeng and Xia, Hongtu and Ning, Xuefei and Yan, Shengen and Liang, Yun and others},
  journal={IEEE Transactions on Computers},
  year={2025}
}

@inproceedings{alizadeh2024llm,
  title={Llm in a flash: Efficient large language model inference with limited memory},
  author={Alizadeh, Keivan and Mirzadeh, Seyed Iman and Belenko, Dmitry and Khatamifard, S and Cho, Minsik and Del Mundo, Carlo C and Rastegari, Mohammad and Farajtabar, Mehrdad},
  booktitle={ACL},
  pages={12562--12584},
  year={2024}
}

@inproceedings{chen2024internvl,
  title={Internvl: Scaling up vision foundation models and aligning for generic visual-linguistic tasks},
  author={Chen, Zhe and Wu, Jiannan and Wang, Wenhai and Su, Weijie and Chen, Guo and Xing, Sen and Zhong, Muyan and Zhang, Qinglong and Zhu, Xizhou and Lu, Lewei and others},
  booktitle={CVPR},
  pages={24185--24198},
  year={2024}
}

@article{zhang2025quant,
  title={Quant-dLLM: Post-Training Extreme Low-Bit Quantization for Diffusion Large Language Models},
  author={Zhang, Tianao and Li, Zhiteng and Yan, Xianglong and Qin, Haotong and Guo, Yong and Zhang, Yulun},
  journal={arXiv preprint arXiv:2510.03274},
  year={2025}
}

@article{xu2025dllmquant,
  title={Dllmquant: Quantizing diffusion-based large language models},
  author={Xu, Chen and Yang, Dawei},
  journal={arXiv preprint arXiv:2508.14090},
  year={2025}
}

@inproceedings{yu2025mquant,
  title={Mquant: Unleashing the inference potential of multimodal large language models via static quantization},
  author={Yu, JiangYong and Zhou, Sifan and Yang, Dawei and Li, Shuoyu and Wang, Shuo and Hu, Xing and Xu, Chen and Xu, Zukang and Shu, Changyong and Yuan, Zhihang},
  booktitle={ACM MM},
  pages={1783--1792},
  year={2025}
}

@inproceedings{lin2024duquant,
  title={Duquant: Distributing outliers via dual transformation makes stronger quantized llms},
  author={Lin, Haokun and Xu, Haobo and Wu, Yichen and Cui, Jingzhi and Zhang, Yingtao and Mou, Linzhan and Song, Linqi and Sun, Zhenan and Wei, Ying},
  booktitle={NeurIPS},
  pages={87766--87800},
  year={2024}
}

@inproceedings{zhang2024magr,
  title={Magr: Weight magnitude reduction for enhancing post-training quantization},
  author={Zhang, Aozhong and Wang, Naigang and Deng, Yanxia and Li, Xin and Yang, Zi and Yin, Penghang},
  booktitle={NeurIPS},
  pages={85109--85130},
  year={2024}
}

@inproceedings{li2024svdquant,
  title={Svdquant: Absorbing outliers by low-rank components for 4-bit diffusion models},
  author={Li, Muyang and Lin, Yujun and Zhang, Zhekai and Cai, Tianle and Li, Xiuyu and Guo, Junxian and Xie, Enze and Meng, Chenlin and Zhu, Jun-Yan and Han, Song},
  booktitle={NeurIPS},
  year={2024}
}

@article{yang2025mmada,
  title={Mmada: Multimodal large diffusion language models},
  author={Yang, Ling and Tian, Ye and Li, Bowen and Zhang, Xinchen and Shen, Ke and Tong, Yunhai and Wang, Mengdi},
  journal={arXiv preprint arXiv:2505.15809},
  year={2025}
}

@article{xu2025qwen2,
  title={Qwen2. 5-omni technical report},
  author={Xu, Jin and Guo, Zhifang and He, Jinzheng and Hu, Hangrui and He, Ting and Bai, Shuai and Chen, Keqin and Wang, Jialin and Fan, Yang and Dang, Kai and others},
  journal={arXiv preprint arXiv:2503.20215},
  year={2025}
}

@inproceedings{xiao2023smoothquant,
  title={Smoothquant: Accurate and efficient post-training quantization for large language models},
  author={Xiao, Guangxuan and Lin, Ji and Seznec, Mickael and Wu, Hao and Demouth, Julien and Han, Song},
  booktitle={ICML},
  pages={38087--38099},
  year={2023}
}

@article{ma2024affinequant,
  title={Affinequant: Affine transformation quantization for large language models},
  author={Ma, Yuexiao and Li, Huixia and Zheng, Xiawu and Ling, Feng and Xiao, Xuefeng and Wang, Rui and Wen, Shilei and Chao, Fei and Ji, Rongrong},
  journal={arXiv preprint arXiv:2403.12544},
  year={2024}
}

@article{jeon2402l4q,
  title={L4Q: Parameter Efficient Quantization-Aware Fine-Tuning on Large Language Models. arXiv 2024},
  author={Jeon, H and Kim, Y and Kim, JJ},
  journal={arXiv preprint arXiv:2402.04902},
  year={2024}
}

@inproceedings{dettmers2022gpt3,
  title={Gpt3. int8 (): 8-bit matrix multiplication for transformers at scale},
  author={Dettmers, Tim and Lewis, Mike and Belkada, Younes and Zettlemoyer, Luke},
  booktitle={NeurIPS},
  pages={30318--30332},
  year={2022}
}

@article{wang2024longllava,
  title={Longllava: Scaling multi-modal llms to 1000 images efficiently via a hybrid architecture},
  author={Wang, Xidong and Song, Dingjie and Chen, Shunian and Zhang, Chen and Wang, Benyou},
  journal={arXiv preprint arXiv:2409.02889},
  year={2024}
}

@article{lyu2023macaw,
  title={Macaw-llm: Multi-modal language modeling with image, audio, video, and text integration},
  author={Lyu, Chenyang and Wu, Minghao and Wang, Longyue and Huang, Xinting and Liu, Bingshuai and Du, Zefeng and Shi, Shuming and Tu, Zhaopeng},
  journal={arXiv preprint arXiv:2306.09093},
  year={2023}
}

@article{sapkota2025multi,
  title={Multi-modal LLMs in agriculture: A comprehensive review},
  author={Sapkota, Ranjan and Qureshi, Rizwan and Hadi, Muhammad Usman and Hassan, Syed Zohaib and Sadak, Ferhat and Shoman, Maged and Sajjad, Muhammad and Dharejo, Fayaz Ali and Paudel, Achyut and Li, Jiajia and others},
  journal={IEEE Transactions on Automation Science and Engineering},
  year={2025}
}

@article{du2024bitdistiller,
  title={Bitdistiller: Unleashing the potential of sub-4-bit llms via self-distillation},
  author={Du, Dayou and Zhang, Yijia and Cao, Shijie and Guo, Jiaqi and Cao, Ting and Chu, Xiaowen and Xu, Ningyi},
  journal={arXiv preprint arXiv:2402.10631},
  year={2024}
}

@article{frantar2022gptq,
  title={Gptq: Accurate post-training quantization for generative pre-trained transformers},
  author={Frantar, Elias and Ashkboos, Saleh and Hoefler, Torsten and Alistarh, Dan},
  journal={arXiv preprint arXiv:2210.17323},
  year={2022}
}

@article{shao2023omniquant,
  title={Omniquant: Omnidirectionally calibrated quantization for large language models},
  author={Shao, Wenqi and Chen, Mengzhao and Zhang, Zhaoyang and Xu, Peng and Zhao, Lirui and Li, Zhiqian and Zhang, Kaipeng and Gao, Peng and Qiao, Yu and Luo, Ping},
  journal={arXiv preprint arXiv:2308.13137},
  year={2023}
}

@inproceedings{bhalgat2020lsq+,
  title={Lsq+: Improving low-bit quantization through learnable offsets and better initialization},
  author={Bhalgat, Yash and Lee, Jinwon and Nagel, Markus and Blankevoort, Tijmen and Kwak, Nojun},
  booktitle={CVPR},
  pages={696--697},
  year={2020}
}

@article{hu2025ostquant,
  title={Ostquant: Refining large language model quantization with orthogonal and scaling transformations for better distribution fitting},
  author={Hu, Xing and Cheng, Yuan and Yang, Dawei and Xu, Zukang and Yuan, Zhihang and Yu, Jiangyong and Xu, Chen and Jiang, Zhe and Zhou, Sifan},
  journal={arXiv preprint arXiv:2501.13987},
  year={2025}
}

@article{lin2025quantization,
  title={Quantization meets dllms: A systematic study of post-training quantization for diffusion llms},
  author={Lin, Haokun and Xu, Haobo and Wu, Yichen and Guo, Ziyu and Zhang, Renrui and Lu, Zhichao and Wei, Ying and Zhang, Qingfu and Sun, Zhenan},
  journal={arXiv preprint arXiv:2508.14896},
  year={2025}
}

@article{nagel2106white,
  title={A white paper on neural network quantization},
  author={Nagel, Markus and Fournarakis, Marios and Amjad, Rana Ali and Bondarenko, Yelysei and Van Baalen, Mart and Blankevoort, Tijmen},
  journal={arXiv preprint arXiv:2106.08295},
  year={2021}
}

@article{liu2025quantization,
  title={Quantization hurts reasoning? an empirical study on quantized reasoning models},
  author={Liu, Ruikang and Sun, Yuxuan and Zhang, Manyi and Bai, Haoli and Yu, Xianzhi and Yu, Tiezheng and Yuan, Chun and Hou, Lu},
  journal={arXiv preprint arXiv:2504.04823},
  year={2025}
}

@inproceedings{penrose1955generalized,
  title={A generalized inverse for matrices},
  author={Penrose, Roger},
  booktitle={Mathematical Proceedings of the Cambridge Philosophical Society},
  pages={406--413},
  year={1955}
}

@article{liu2025mtp,
  title={L-MTP: Leap Multi-Token Prediction Beyond Adjacent Context for Large Language Models},
  author={Liu, Xiaohao and Xia, Xiaobo and Zhao, Weixiang and Zhang, Manyi and Yu, Xianzhi and Su, Xiu and Yang, Shuo and Ng, See-Kiong and Chua, Tat-Seng},
  journal={NeurIPS},
  year={2025}
}
